\documentclass{article}

 \usepackage[final]{nips_2018}

\usepackage[utf8]{inputenc} 
\usepackage[T1]{fontenc}    
\usepackage{hyperref}       
\usepackage{url}            
\usepackage{booktabs}       
\usepackage{amsfonts}       
\usepackage{nicefrac}       
\usepackage{microtype}      
\usepackage{graphicx}

\usepackage{nth}
\usepackage{amsmath}
\usepackage{amsthm}
\usepackage{amsfonts}
\usepackage{amssymb}
\usepackage{amssymb}
\usepackage{overpic}
\usepackage{multirow}
\usepackage{xcolor}

\newcommand{\blambda}{\boldsymbol\lambda}

\usepackage{booktabs}
\usepackage{subfigure}
\usepackage{xr}
\usepackage{hyperref}
\newtheorem{theorem}{Theorem}
\newtheorem{corollary}{Corollary}[theorem]
\DeclareMathOperator{\atantwo}{atan2}

\title{Bridging the Gap Between Spectral and Spatial Domains in Graph  
Neural Networks }

\author{
  Muhammet~Balcilar
   \thanks{\texttt{muhammetbalcilar@gmail.com}}, ~ 
   Guillaume~Renton,
   Pierre~H\'eroux,
   Benoit~Ga\"uz\`ere, 
   S\'ebastien~Adam,    
   Paul~Honeine  
   \\ \\
 LITIS Lab, University of Rouen Normandy \\
 Rouen, FRANCE \\
}

\begin{document}
\maketitle

\begin{abstract}
This paper aims at revisiting Graph Convolutional Neural Networks by bridging the gap between spectral and spatial design of graph convolutions. We theoretically demonstrate some equivalence of the graph convolution process regardless it is designed in the spatial or the spectral domain. The obtained general framework allows to lead a spectral analysis of the most popular ConvGNNs, explaining their performance and showing their limits. Moreover, the proposed framework is used to design new convolutions in spectral domain with a custom frequency profile while applying them in the spatial domain. We also propose a generalization of the depthwise separable convolution framework for graph convolutional networks, what allows to decrease the total number of trainable parameters by keeping the capacity of the model. To the best of our knowledge, such a framework has never been used in the GNNs literature. Our proposals are evaluated on both transductive and inductive graph learning problems. Obtained results show the relevance of the proposed method and provide one of the first experimental evidence of transferability of spectral filter coefficients from one graph to another. Our source codes are publicly
available at: \\ {\color{blue}\url{https://github.com/balcilar/Spectral-Designed-Graph-Convolutions}}
\end{abstract}

\keywords{Graph Convolutional Neural Networks, Spectral Graph Filter.}

\section{Introduction  }
Over the past decade, Deep Learning, and more specifically Convolutional Neural Networks (CNNs) and Recurrent Neural Networks (RNNs), had a strong impact in various applications of machine learning, such as image recognition [1] and speech analysis [2]. These successes have mostly been achieved on sequences or images, i.e. on data defined on grid structures which benefit from linear algebra operations in Euclidean spaces. However, there are many domains where data (e.g. social networks, molecules, knowledge graph) cannot be trivially encoded into an Euclidean domain, but can be naturally represented as graphs.

This explains the recent challenge tackled by the machine learning community which consists in transposing the deep learning paradigm into the world of graphs. The objective is to revisit Neural Networks to operate on graph data, in order to benefit from the representation learning ability. In this context, many Graph Neural Networks (GNNs) have been recently proposed in the literature of geometric learning  [3,4,5,6]. GNNs are Neural Networks that rely on the computation of hidden representations of nodes using information carried by the whole graph. In contrast to conventional Neural Network, where the architecture of the network is related to the known and invariant topology of the data (e.g. a 2-D grid for images), the node features of GNNs are propagated according to the graph topology.

Among GNNs, Convolutional GNNs (ConvGNNs) aim to mimic the simple and efficient solution provided by CNN to extract features through a weight-sharing strategy along the presented data. In images, a convolution relies on the computation of a weighted sum of neighbor’s features and weight-sharing is possible thanks to the neighbor relative positions. With graph-structured data, designing such a convolution process is not straightforward. First, there is a variable and unbounded number of neighbors, avoiding the use of a fixed sized window to compute the convolution. Second, no order exists on node neighborhood. As a consequence, one may first redefine the convolution operator to design a ConvGNN.

As in images, a graph convolution process corresponds to the multiplication of a convolution kernel with the corresponding node feature vectors, followed by a sum or a mean rule. In the literature, there are some instances of trainable and non-trainable convolution kernels for graphs. Regardless if the convolution kernels are trainable or not, and according to the convolution theorem, two strategies have been investigated to design filter kernels, based either on the spectral or the spatial domains. 

Spectral-based convolution filters are defined from a graph signal processing point of view. In a nutshell, a basis is defined by the eigendecomposition of the graph Laplacian matrix. This allows to define the graph Fourier transform, and thus the graph filtering operators. The original form of the spectral graph convolution (non-parametric) can be defined by a function of frequency (eigenvalues). Therefore, this method can theoretically extract information on any frequency. However, despite the solid mathematical foundations borrowed from the signal processing literature, such approaches suffer from (i) a large computational burden induced by the forward/inverse graph Fourier transform, (ii) being spatially non-localized and (iii) the transferability problem, i.e., filters designed using a given graph cannot be applied on other graphs. To alleviate these issues, some approaches based on parameterization using B-spline [7], Chebyshev polynomials [8] and Cayley polynomials [9] have been proposed.
However, these approaches cannot use custom designed frequency response convolution, but only the one determined by B-spline, Chebyshev or Cayley polynomials. That means these methods cannot extract information on some custom band.

The second strategy is the spatial-based convolution, which is an extension of the conventional Euclidean convolution (e.g. 2D convolution in CNN), by aggregating nodes neighborhood information. Such convolutions have been very attractive due to their less computational complexity, their localized property and their transferability. Spatial-designed graph convolutions have to be a function of some spatial properties of graphs, such as adjacency, Laplacian or degree matrix combined with feature of connected nodes, and edge features. However, since they are designed in the spatial domain, their spectral behavior is not taken into account. We will show in the following that most of the existing spatial-designed convolutions are essentially low-pass filters. As a consequence, they do not have the ability to extract useful information on high frequency or some certain frequency bands. Yet, considering high-frequency information may be intuitively useful for some real-world problems where information localized on particular nodes have a strong influence on graph's property. For instance, molecular toxicity can be induced by some pharmacophores, i.e., particular subparts of molecule, which can consist in only one atom. Using only low-pass filters on such molecules will diffuse this discriminant information within the whole graph whereas a high-pass filter may help to highlight this useful difference.

\subsection*{Contributions}

In this paper, we bridge the gap between spectral and spatial domains for ConvGNNs. Our first contribution consists in demonstrating the equivalence of convolution processes regardless if they are designed in the spatial or the spectral domain. Taking advantage of this result, our second contribution is to  provide a spectral analysis of existing graph convolutions for four popular ConvGNNs, known as 
GCN [10], ChebNet [8], CayleyNet [9] and Graph Attention Networks (GAT) [11]. Using these results, our third contribution is to  design new convolutions in the spectral domain with a custom frequency profile
that provides a better convolution process. In this context, we also propose a spectral-designed multi-convolution method under the depthwise separable convolution framework. To the best of our knowledge, such a framework has never been used in the GNNs literature. It allows to decrease the total number of trainable parameters by keeping the variability capacity of the model at a maximum level. 

Our proposal is assessed on both transductive and inductive learning problems  [12]. 
In both settings, we show the relevance of the proposed method on well-known public benchmark datasets. Especially, the 
success of the proposed method on inductive problems provides one of the first experimental evidence of transferability of spectral filter coefficients from one graph to another.

The remainder of this paper is organized as follows. In Section~2, we introduce ConvGNNs and we review existing approaches. Then, Section~3 describes the three main contributions mentioned above. Section~4 presents a series of experiments and results which validate our propositions. Finally, Section~5 is dedicated to the conclusion. 

\section{ConvGNN: Problem Statement and State of the Art}

\subsection{Graph Learning Problems}
\label{secGLP}
Let $\mathcal{G}$ be a set of graphs, where each graph $G^{(k)}$ has $n_k$ nodes and an arbitrary number of edges. Node-to-node connectivity in $G^{(k)}$ is given by the adjacency matrix $A^{(k)}$. For unweighted graphs, $A^{(k)} \in \{0,1\}^{n_k \times n_k}$, while for weighted graphs, $A^{(k)} \in \mathbb{R}^{n_k \times n_k}$. In this paper, we consider undirected attributed graphs. Hence, $A^{(k)}$ is symmetric and features are defined on nodes by $X^{(k)} \in \mathbb{R}^{n_k \times f_0}$, with $f_0$ the length of feature vectors.

In the literature, there are three different types of learning problems on graphs. The first one is the single graph node classification or regression problem. In this case, $\mathcal{G}$ is reduced to a single graph denoted $G$, with $n$ nodes. Some of the nodes are labeled for training and the task is to predict the labels of unlabeled nodes. For a classification problem, the output would be represented by $\mathcal{Y} \in \{ 0, 1\}^{n \times n_c}$, i.e., a one-hot class encoding of the $n_c$ possible classes for each node. For a node regression problem, the output would be $\mathcal{Y} \in \mathbb{R}^n$. The second type of problems is multi-graph node classification or regression problem. In such cases, the output is defined as a set of $\mathcal{Y}^{(k)} \in \{ 0, 1\}^{n_k \times n_c}$ for classification or  $\mathcal{Y}^{(k)} \in \mathbb{R}^{n_k}$ for regression. The last type is the entire graph classification or regression problem, in which case the output must be $\mathcal{Y}^{(k)} \in \{ 0, 1\}^{n_c}$ 
or $\mathcal{Y}^{(k)} \in \mathbb{R}$  for classification and regression problems, respectively. 

Problems of the first type are transductive problems, while problems of the two last types are inductive since test data are completely unknown during training.

\subsection{Literature review}

For reviewing 
ConvGNNs, we use the classical ``spectral vs. spatial" 
dichotomy~[13]. Beyond, we propose for this review a third category called Spectral-Rooted Spatial Convolutions which gathers recent and efficient methods that take their foundations in the spectral domain, but apply them in the spatial one, without computing the graph Fourier transform.

\subsubsection{Spectral ConvGNN}
\label{sec2.2}

Spectral ConvGNNs rely on the spectral graph
theory~[14]. In this framework, signal on graphs are filtered using eigendecomposition of graph Laplacian [15]. A graph Laplacian is defined by $L=D-A$ (or $L=I-D^{-1/2}AD^{-1/2}$ for the normalized version), where $A$ is the adjacency matrix,  $D \in \mathbb{R}^{n_k \times n_k}$ is the diagonal degree matrix with entries $D_{i,i} = \sum_j A_{j,i}$ and $I$ is the identity matrix. Since the Laplacian is positive semidefinite, 
it can be decomposed into $L=U \Sigma U^T$ where $U$ is the eigenvectors matrix and $\Sigma=diag(\blambda)$ where $\blambda$ denotes the vector of the positive eigenvalues. The graph Fourier transform of any unidimensional signal on graph is defined by
$x_{ft}=U^\top x$ and its inverse is given by $x=U x_{ft}$. By transposing the convolution theorem to graphs, the spectral filtering in the frequency domain can be defined by 
\begin{equation}
  \label{eq:Eqfilter}
  x_{filtered} =  U diag(\digamma(\blambda)) U^\top x,
\end{equation}
where $\digamma(\blambda)$ is the desired filter function applied to the eigenvalues $\blambda$. 
As a consequence, a graph convolution layer in spectral domain can be written by a sum of filtered signals followed by an activation function as in [7], namely 
\begin{equation}
  \label{eq:Eq4}
  H_j^{(l+1)} = \sigma \left( \sum_{i=1}^{f_l} U diag (F_{i,j,l}) U^\top
  H_i^{(l)} \right), 
\end{equation}

for all $j \in \{ 1, \dots, f_{l+1}\}$. Here, $\sigma$ is the activation function such as \verb|RELU| (REctified Linear Unit), $H_i^{(l)}$ is the $i$-th feature vector of the $l$-th layer,  $F_{i,j,l} \in \mathbb{R}^{n}$
is the corresponding  weight vector whose size is the number
of eigenvectors (also $n$, the number of nodes). A spectral ConvGNN based on \eqref{eq:Eq4} seeks to tune the trainable parameters $F_{i,j,l}$, as proposed in [16] for the single-graph problem. 

A first drawback is the necessity of Fourier and inverse Fourier transform by matrix multiplication of $U$ and $U^T$. Another drawback occurs when generalizing the approach to multi-graph learning problems. Indeed, the $k$-th element of the vector $F_{i,j,l}$  weights the contribution of the $k$-th  eigenvector to the output. Those weights are not shareable between graphs of different sizes, which means a different length of $F_{i,j,l}$ is needed. Moreover, even though the graphs have the same number of nodes, their eigenvalues will be different if their structures differ. As a consequence, a given weight $F_{i,j,l}$ may correspond to different eigenvalues in different graphs. 

To overcome these issues, a few spatially-localized filters have been defined such as cubic B-spline parameterization [7] and polynomial parameterization [8]. With such approaches, trainable parameters are defined by:
\begin{equation}
\label{eq:Eq4b}
    F_{i,j,l} = B \, \left[W_{i,j}^{(l,1)}, \dots,W_{i,j}^{(l,S)} \right]^\top,
\end{equation}
where $B \in \mathbb{R}^{n \times S}$ is the initial designed matrix and $W^{(l,s)}$ is the trainable matrix for the $l$-th layer's $s$-th convolution kernel, $W_{i,j}^{(l,s)}$ is the $(i,j)$-th entry of $W^{(l,s)}$ and $S$ is the desired number of convolution kernels. Each column in $B$ is designed as a function of eigenvalues, namely $B_{i,j}=(\digamma_j(\lambda_i))$.
In the polynomial case, each column of $B$ 
is power of eigenvalues starting at $0$-th and ending at $(S-1)$-th power
. In the cubic B-spline case, the $B$ matrix encodes the cubic B-spline coefficients [7].   
A very recent ConvGNN named CayleyNet parameterizes trainable coefficients by $F_{i,j,l}=[g_{i,j,l}(\lambda_1,h),...,g_{i,j,l}(\lambda_n,h)]^{\top}$, where $h$ is a scale parameter to be learned, $\lambda_n$ is the $n$-th eigenvalue, and  $g$ is a spectral filter function defined as follows in [9]:
\begin{equation}
  \label{eq:Eqcy}
  g(\lambda,h)=c_0+2 Re\left( \sum_{k=1}^{r}c_k \left(\frac{h\lambda-\textbf{i}}{h\lambda+\textbf{i}}\right)^k  \right)
\end{equation}
\noindent where $\textbf{i}^2=-1$, $Re(\cdot)$ is the function returning the real part
, $c_0$ is a real trainable coefficient, and for $k=1,\ldots ,r$, $c_k$ are the complex trainable coefficients. The CayleyNet parameterization takes also the form \eqref{eq:Eq4b}, as shown in Appendix B.

\subsubsection{Spatial ConvGNN}

Spatial ConvGNNs can be generalized as propagation of node features to the neighborhood nodes followed by activation function, of the form 
\begin{equation}
  \label{eq:Eq3}
  H^{(l+1)} = \sigma \Big( \sum_s C^{(s)} H^{(l)} W^{(l,s)} \Big),
\end{equation}
where $H^{(l)} \in \mathbb{R}^{n \times f_l}$ is the $l$-th layer's feature matrix with $n$ nodes and $f_l$ features, $s$  indexes the convolution kernels, $C^{(s)}$ is the convolution kernel that defines how the node features are propagated to the neighborhood nodes, $W^{(l,s)} \in \mathbb{R}^{f_l \times f_{l+1}}$ is the trainable weight matrix that maps the $f_l$-dimensional features into $f_{l+1}$ dimensions. \figurename~\ref{fig:detail_gcnn} provides a detailed schematic of graph convolution layer on a sample graph signal.
The selection of convolution kernels defines the method in the literature. The vanilla version uses a single convolution kernel with $C=A+I$. Such a spatial ConvGNN has an effect of low-pass filtering, since it applies the same coefficients to all neighbors and to the node itself. High-pass filters can be obtained by differentiating the weight matrices used to compute neighbors and self-contributions [17]. In such a case, the convolution process is given by $C^{(1)}=A$ and $C^{(2)}=I$.  

\begin{figure}
  \centering
  \includegraphics[width=.95\textwidth]{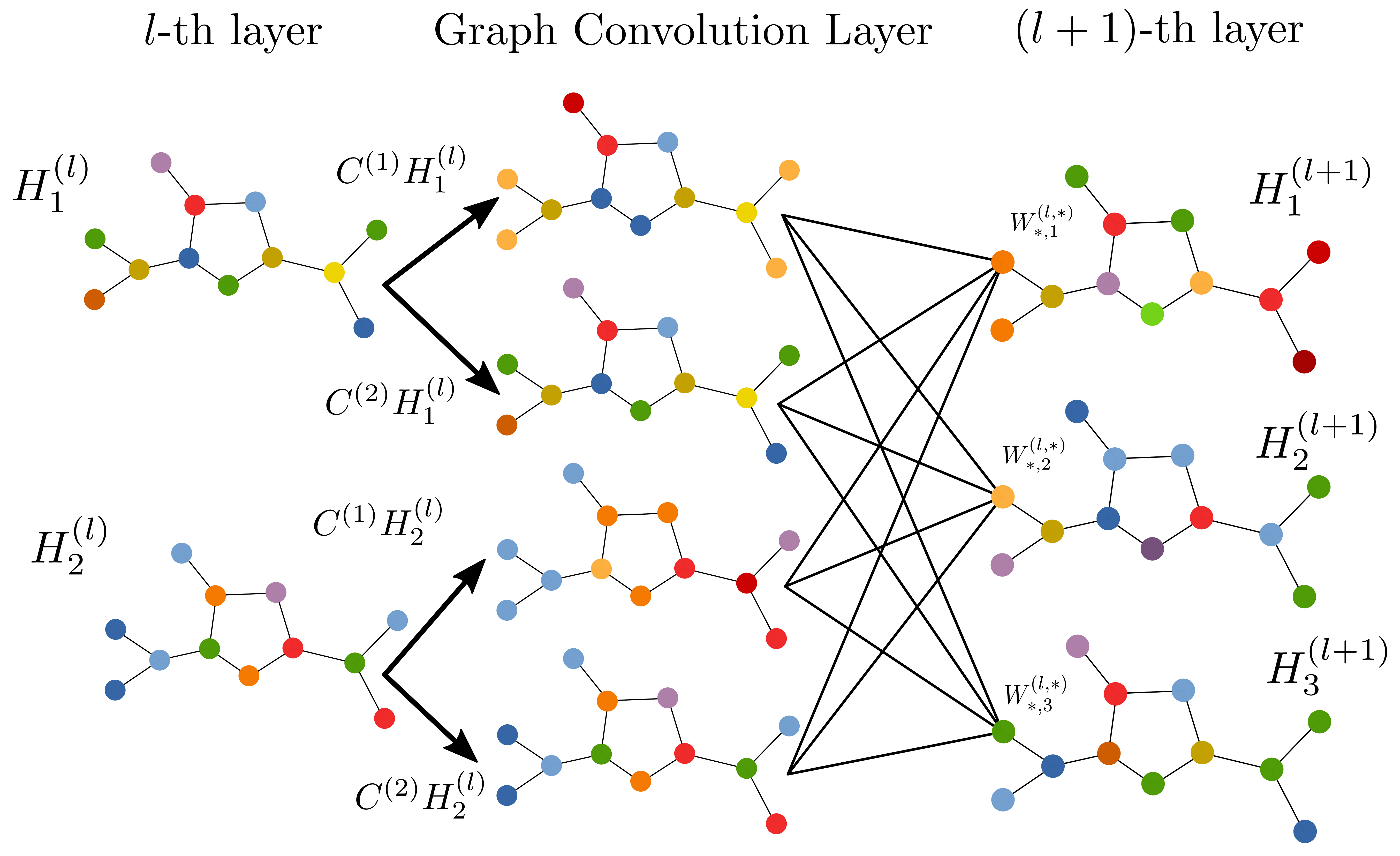}
  \caption{Schematic of the GCN layer defined in \eqref{eq:Eq3}. The graph has 12 nodes and 12 edges. Each node has a 2-length feature vector $H_1^{(l)}$ and $H_2^{(l)}$ represented by colors. The second layer has a 3-length feature vector, denoted $H_1^{(l+1)}$, $H_2^{(l+1)}$ and $H_3^{(l+1)}$. Two convolution kernels $C^{(1)}$ and $C^{(2)}$ are used. 
  This architecture has 12 trainable parameters, omitting biases.}
  \label{fig:detail_gcnn}
  \medskip
\end{figure}

Some solutions have been proposed to overcome the limitations of using only low-pass and high-pass filters. 
If nodes have discrete labels (unless the  node's degree can be used as discrete feature), weights can be shared by the neighbors whose labels are the same [18]. 
Another method consists in defining an ordering on nodes included within the receptive field of convolution, and sharing the coefficients according to this reordering~[19]. The reordering process is called canonical node reordering. 
A similar sharing approach, based on reordered neighbors, was presented in [20]. The difference is that the reordering is computed according to the absolute correlation of features to the center node. 
A different spatial-designed method proposed in [21] considers a diffusion process on the graph using random walks. This allows to induce variability on output signal by applying random walks of different lengths to the different features. 

All aforementioned spatial graph convolutions use fixed-design matrices $C^{(s)}$ and variability is induced by $W^{(l,s)}$ in \eqref{eq:Eq3}. Other methods use trainable convolution kernels in order to make the convolutions more productive in terms of output signal frequency profiles, such as graph attention networks [11,12], MoNet [23] and SplineCNN [24]. The attention mechanism tunes each element of the convolution kernel of the $l$-th layer $C^{(l,s)}$, which is defined as a function of connected nodes features and some trainable parameter
\begin{equation}
  \label{eq:gat1}
  C^{(l,s)}_{i,j}=f\left(H^{(l)}_i,H^{(l)}_j,W_{AT}^{(l,s)}\right),
\end{equation}
where $H^{(l)}_i \in \mathbb{R}^{f_l}$ is the $i$-th node's feature vector for layer $l$, 
$W_{AT}^{(l,s)}$ encodes the trainable parameter of the $s$-th convolution kernel for layer $l$ and $f$ is some element-wise function to be selected.
In this case, since convolution kernels are learned through the $W_{AT}^{(l,s)}$ of \eqref{eq:gat1}, the trainable parameters  $W^{(l,s)}$ of \eqref{eq:Eq3} can be defined as the identity matrix, or other trainable parameters that may be shared with $W_{AT}^{(l,s)}$.
The most influential attention mechanism applied on graph data, called GAT [11], uses multi-attention weights (denoted as multi-support convolution kernels), 
with 
\begin{equation}
  \label{eq:gat2}
  f \left( H^{(l)}_i, H^{(l)}_j, W_{AT}^{(l,s)} \right) \!=\! \verb|softmax|_j \left( \sigma(\mathbf{a}[\mathbf{W} H_i^{(l)} || \mathbf{W} H_j^{(l)}]) \right),
\end{equation}
where two 
linear transformations
are considered by elements of general trainable parameter
  set $W_{AT}^{(l,s)} \!=\! \{\mathbf{a},\mathbf{W}\}$, with $\mathbf{a}$ being a weight vector. The operator $||$ is the concatenation
operator, $\sigma$ corresponds to the $\verb|LeakyReLU|$
function~[25] and $\verb|softmax|_j$ is the normalized exponential function that uses all neighbors of $i$-th node to normalize edge of $i$-th to $j$-th node. In convolution layer's output calculation \eqref{eq:Eq3}, GAT proposes to use the same
parameters $\mathbf{W}$. 
The main limitation
of this  method is the use of a very small context, limited to the features of the pair of nodes, to determine the
intensity of the attention. 
Dual-Primal Graph CNN (DPGCNN)
[26] extends this approach by defining attention using new features
computed from the
neighborhood of each node of the pair, hence using a larger
context. 

Since the methods mentioned above are defined in the spatial domain,
they do not provide any analysis of their frequency spectrum of
filters. Moreover, their frequency responses will be different for
different graphs. Besides, they need more multi-support (attention or
sub-layer weights) to produce high variability output, which drastically increases the number of trainable parameters of the model.

\subsubsection{Spectral-rooted Spatial Convolutions}

As said before, some methods have recently been proposed to get rid of the computation burden of graph Fourier and inverse graph Fourier transforms, while still taking their foundations in the spectral domain. These solutions rely on the approximation of a spectral graph convolution proposed in~[27], 
based on the Chebyshev polynomial expansion of the scaled graph Laplacian. Accordingly, the first two Chebyshev kernels are $C^{(1)}=I$ and
$C^{(2)}=2L/\lambda_{\max}-I$ and the remaining kernels are defined by 
\begin{equation}
  \label{eq:chebeq}
  C^{(k)}=2C^{(2)}C^{(k-1)} - C^{(k-2)}.
\end{equation}
Researchers have shown that any desired filter can be written as a linear combination of these kernels [27].
ChebNet is the first method that used these kernels in ConvGNN [8]. 

One major extension and simplification of the Chebyshev polynomial expansion method is Graph Convolution Network (GCN) [10]. 
GCN uses the subtraction of the second Chebyshev kernels from the first one under the assumption of $\lambda_{\max}=2$ and $L$ is the normalized graph Laplacian.  
However, instead of using this subtracted kernel, they used re-normalization trick and defined
the final single kernel by:

\begin{equation}
  \label{eq:Eq14}
  C = \widetilde{D}^{-1/2}\widetilde{A} \widetilde{D}^{-1/2},
\end{equation}
with $\widetilde{D}_{i,i} = \sum_j \widetilde{A}_{i,j}$ and $\widetilde{A} = ( A + I)$ the adjacency matrix with added self-connections. 
This approach influenced many other contributions. The method
described in [28] directly uses this convolution but
changes the network architecture by adding a fully connected layer as the
last layer. 
The MixHop algorithm [29] uses the \nth{2} or \nth{3} powers of the same convolution. 

The methods described in this section are quite different from pure spatial and pure spectral convolutions. They are not designed by using eigenvalues, 
but are implicitly designed as a function of structural information (adjacency, Laplacian) and perform convolution in spatial domain as how all spatial convolutions do. However, their frequency profiles are stable for different arbitrary graphs as how spectral convolutions do.  This aspect will be theoretically and experimentally illustrated in the following sections.  

\section{Bridging Spatial and Spectral ConvGNN}

This section presents the main theoretical contributions of this paper. First, we provide  a theoretical analysis demonstrating that parameterized spectral ConvGNNs can be implemented as spatial ConvGNNs when they use a fixed frequency profile matrix $B$. Then, using this result, some state-of-the-art GNNs described in the previous section are analyzed from a spectral point of view. This analysis provide a better understanding on these convolutions and reveal their problematic sides. Finally, we propose a new method that fully exploits spectral graph convolution capabilities, called Depthwise Separable Graph Convolution Network.

\subsection{Theoretical analysis}

\begin{theorem}
  \label{Th:th1}
  Spectral ConvGNN  parameterized with fixed frequency profiles matrix $B$ of entries $B_{i,j}=\digamma_j(\lambda_i)$, defined as 
  \begin{equation}
  \label{eq:Eq5b}
  H_j^{(l+1) } \!\!=\! \sigma \Big( \sum^{f_l}_{i=1} U diag \Big( B \, \Big[W _{i,j}^{(l,1)}, \dots, W_{i,j}^{(l,S)} \Big]^{\!\!\top} \Big) U^\top
  H_i^{(l)} \Big), 
\end{equation}
is a particular case of spatial ConvGNN, defined as 
\begin{equation}
  \label{eq:Eq3b}
  H^{(l+1)} = \sigma \Big( \sum_s C^{(s)} H^{(l)} W^{(l,s)} \Big),
\end{equation}
with the convolution kernel set to
\begin{equation}
  \label{eq:Eq9a}
  C^{(s)} = U diag (\digamma_s(\blambda)) U^\top,
\end{equation}
where the columns of $U$ are the eigenvectors of the studied graph, 
$\sigma$ is the activation function, $H^{(l)} \in \mathbb{R}^{n \times f_l}$ is the $l$-th layer's feature matrix with $f_l$ features, $H_i^{(l)}$ is the $i$-th column of $H^{(l)}$,  $B \in \mathbb{R}^{n \times S}$ is an apriori designed matrix for each graph's
eigenvalues, and $\digamma_s(\blambda)$ is the $s$-th column of $B$. Both $W^{(l,s)}$ and $S$ are defined in \eqref{eq:Eq4b}. 
\end{theorem}

\begin{proof}
First, let us expand the matrix $B$ and rewrite it as the sum of its columns, denoted $\digamma_1(\blambda), \dots, \digamma_S(\blambda) \in \mathbb{R}^n$:
\begin{equation}
  \label{eq:Eq6}
  H_j^{(l+1) } = \sigma\left( \sum^{f_l}_{i=1} U diag \Big( \sum_{s=1}^S W
  _{i,j}^{(l,s)}\digamma_s(\blambda)  \Big) U^\top  H_i^{(l)} \right).
\end{equation}
Now, we distribute $U$ and $U^\top$ over the inner summation: 
\begin{equation}
  \label{eq:Eq7}
  H_j^{(l+1) } = \sigma\left( \sum_{s=1}^S  \sum^{f_l}_{i=1} U diag \Big( W
  _{i,j}^{(l,s)}\digamma_s(\blambda) \Big) U^\top H_i^{(l)}  \right).
\end{equation}
Then, 
we take out the scalars $W_{i,j}^{(l,s)}$ of the $diag$ operator:
\begin{equation}
  \label{eq:Eq8}
  H_j^{(l+1) } = \sigma\left( \sum_{s=1}^S \sum^{f_l}_{i=1}
  W_{i,j}^{(l,s)} U diag (\digamma_s(\blambda))U^\top H_i^{(l)} \right).
\end{equation}
Let us define a convolution operator $C^{(s)} \in \mathbb{R}^{n \times n}$ as:
\begin{equation}
  \label{eq:Eq9}
  C^{(s)} = U diag (\digamma_s(\blambda)) U^\top.
\end{equation}
Using \eqref{eq:Eq8} and \eqref{eq:Eq9}, we have thus:
\begin{align}
  \label{eq:Eq10}
  H_j^{(l+1) }
  & = \sigma\left(  \sum^{f_l}_{i=1} \sum_{s=1}^S 
                 W_{i,j}^{(l,s)} C^{(s)}  H_i^{(l)} \right).
\end{align}
Then, each term of the sum over $s$ corresponds to  a matrix $ H^{(l+1)} \in \mathbb{R}^{n \times f_{l+1}}$ with
\begin{equation}
  \label{eq:Eq11}
  H^{(l+1)} = \sigma \left( C^{(1)}H^{(l)}W^{(l,1)} + \dots + C^{(S)}H^{(l)}W^{(l,S)} \right),
\end{equation}
with 
$H^{(l)} = [H_1^{(l)}, \ldots , H_{f_l}^{(l)}]$. 
We get by grouping the terms:
\begin{equation}
  \label{eq:Eq12}
  H^{(l+1)} = \sigma \left( \sum_{s=1}^S C^{(s)}H^{(l)}W^{(l,s)} \right),
\end{equation}
which corresponds to \eqref{eq:Eq3b}.
Therefore, \eqref{eq:Eq5b} corresponds to \eqref{eq:Eq3b} with $C^{(s)}$ defined as \eqref{eq:Eq9}.
\end{proof}

This theorem is general, since it covers many well-known spectral ConvGNNs, 
such as non-parametric spectral graph convolution [16], polynomial parameterization [8], cubic B-spline parameterization [7] and CayleyNet [9].

From Theorem~\ref{Th:th1}, designing a graph convolution either in spatial or in spectral domain is equivalent. Therefore, Fourier calculations are not necessary when convolutions are parameterized by an initially designed matrix $B$. 
Using that relation, it is not difficult to show the spatial equivalence of non-parametric spectral graph convolution defined in \eqref{eq:Eq4}. It can be written in spatial domain with $B=I$ in \eqref{eq:Eq4b}. It thus corresponds to~\eqref{eq:Eq3b} where each convolution kernel is defined by $C^{(s)}=U_s U_s^{\top}$, where $U_s$ is the $s$-th eigenvector. 

\subsection{Spectral Analysis of Existing Graph Convolutions}
\label{secSAGC}

This section aims at providing a deeper understanding of the graph convolution process through an analysis of existing GNNs in the spectral domain. 
To the best of our knowledge, no one has led such an analysis concerning graph convolutions in the literature. In this section, we show how it can be done on four well-known graph convolutions: ChebNet~[8], CayleyNet~[9], GCN~[10] and GAT~[11]. This analysis is led using the following corollary of Theorem~\ref{Th:th1}. 

\begin{corollary}
\label{cor:frequency_profile}
The frequency profile of any given graph convolution kernel $C^{(s)}$ can be defined in spectral domain by the vector 
\begin{equation}
  \label{eq:Eq13t}
  \digamma_s(\blambda) = diag^{-1} ( U^\top C^{(s)} U ).
\end{equation}
\end{corollary}
\begin{proof}
By using \eqref{eq:Eq9a} from Theorem~\ref{Th:th1}, we can obtain a spatial convolution kernel $C^{(s)}$ whose frequency profile is $\digamma_s(\blambda)$. Since the eigenvector matrix is orthonormal (i.e., $U^{-1}=U^\top$), we can extract $\digamma_s(\blambda)$, which yields \eqref{eq:Eq13t}.
\end{proof}

We denote the matrix $\digamma_s=U^\top C^{(s)} U$ as the full frequency profile of the convolution kernel $C^{(s)}$, and $\digamma_s(\blambda)=diag(\digamma_s)$ as the standard frequency profile of the convolution kernel. The full frequency profile includes all eigenvector-to-eigenvector pairs contributions. Standard frequency profile just includes each eigenvector's self-contribution. 

To show the frequency profiles of some well-known graph convolutions, we used three graphs. The first one corresponds to a 1D signal encoded as a regular circular line graph with 1001 nodes. The second and third ones are the Cora and Citeseer reference datasets, which consist of one single graph with respectively 2708 and 3327 nodes~[12]. Basically, each node of these graphs is labeled by a vector, and edges are unlabeled and undirected. These two graphs will be described in details in Section~\ref{sec:xp}.

\subsection*{ChebNet}

\begin{figure}
\centering
\includegraphics[width=.55\textwidth]{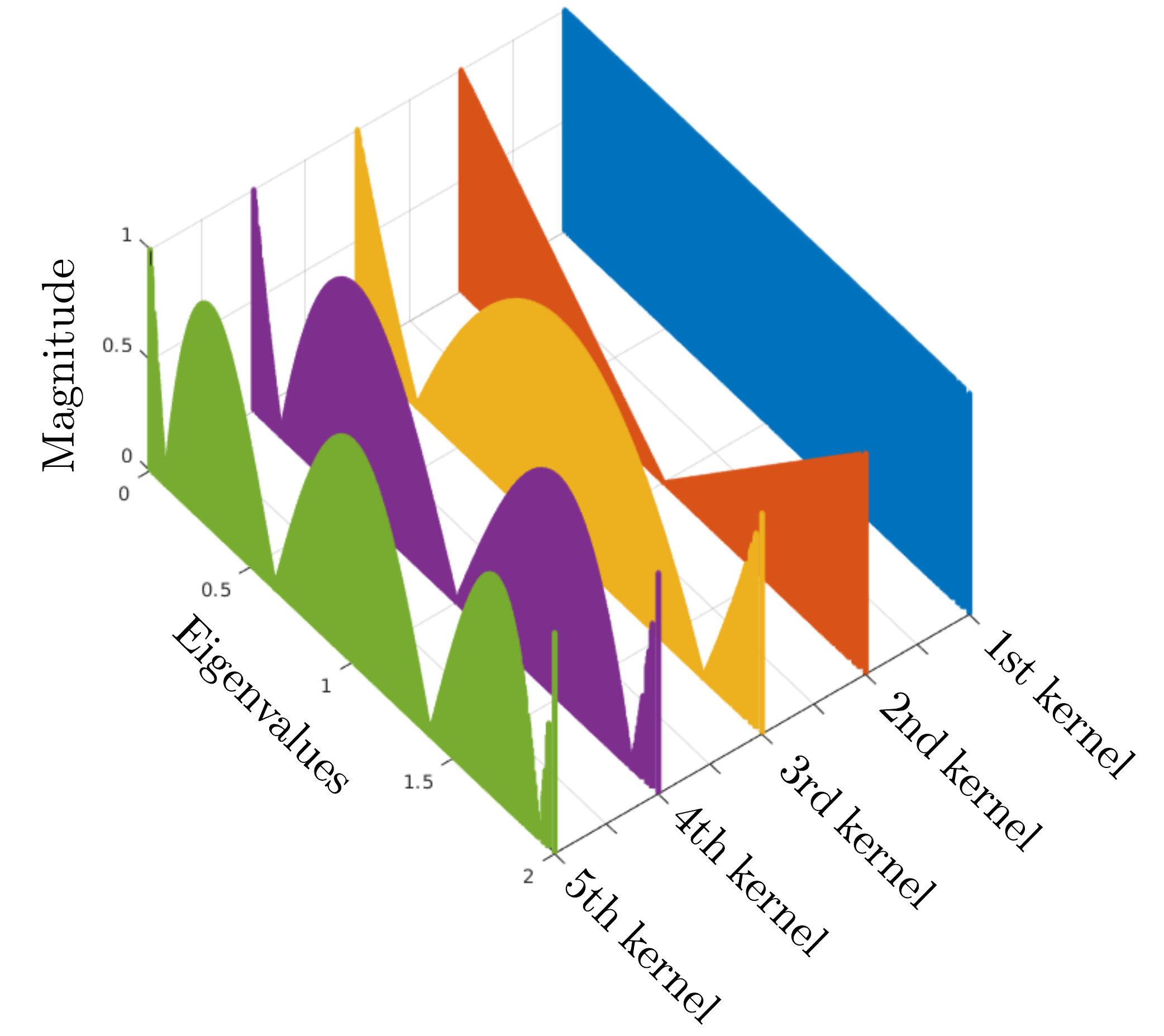}
\medskip
\caption{Standard frequency profiles of first 5 Chebyshev convolutions.
}
\label{fig:chebfreq}
\bigskip
\end{figure}

After computing the kernels of ChebNet by \eqref{eq:chebeq}, Corollary~\ref{cor:frequency_profile} can be used to obtain their frequency profiles. As shown in Appendix A, the first two kernel frequency profiles of ChebNet are $\digamma_{1}(\blambda)={\bf1}$ and $\digamma_{2}(\blambda)=2\blambda/\lambda_{\max}-{\bf1}$, where {\bf1} is the vector of ones. Since $\lambda_{\max}=2$ for all three graphs, 
we get $\digamma_{2}(\blambda)=\blambda-{\bf1}$. The third one and following kernel frequency profiles can also be computed using $\digamma_{k}(\blambda)=2\digamma_{2}(\blambda)\digamma_{k-1}(\blambda) - \digamma_{k-2}(\blambda)$, leading to $\digamma_{3}(\blambda)=\blambda^2-4\blambda+{\bf1}$ for example for the third kernel. 
The resulting 5 frequency profiles are shown in \figurename~\ref{fig:chebfreq} (in absolute value).  
Since the full frequency profiles consist of zeros outside the diagonal, they are not illustrated.

Analyzing the frequency profile of ChebNet, one can argue that the convolutions mostly cover the spectrum. However, none of the kernels focuses on some certain parts of the spectrum. As an example, the second kernel is mostly a low-pass and high-pass filter and stops the middle band, while the third one passes very high, very low and middle bands, but stops almost first and third quarter of the spectrum. Therefore, if the relation between input-output pairs can be figured out by just a low-pass, high-pass or some specific band-pass filter, a high number of convolution kernels is needed. However, in the literature, only 2 or 3 kernels are generally used for experiments~[8,10]. 

\subsection*{CayleyNet}

\begin{figure}[!t]
\centering
\includegraphics[width=.55\textwidth]{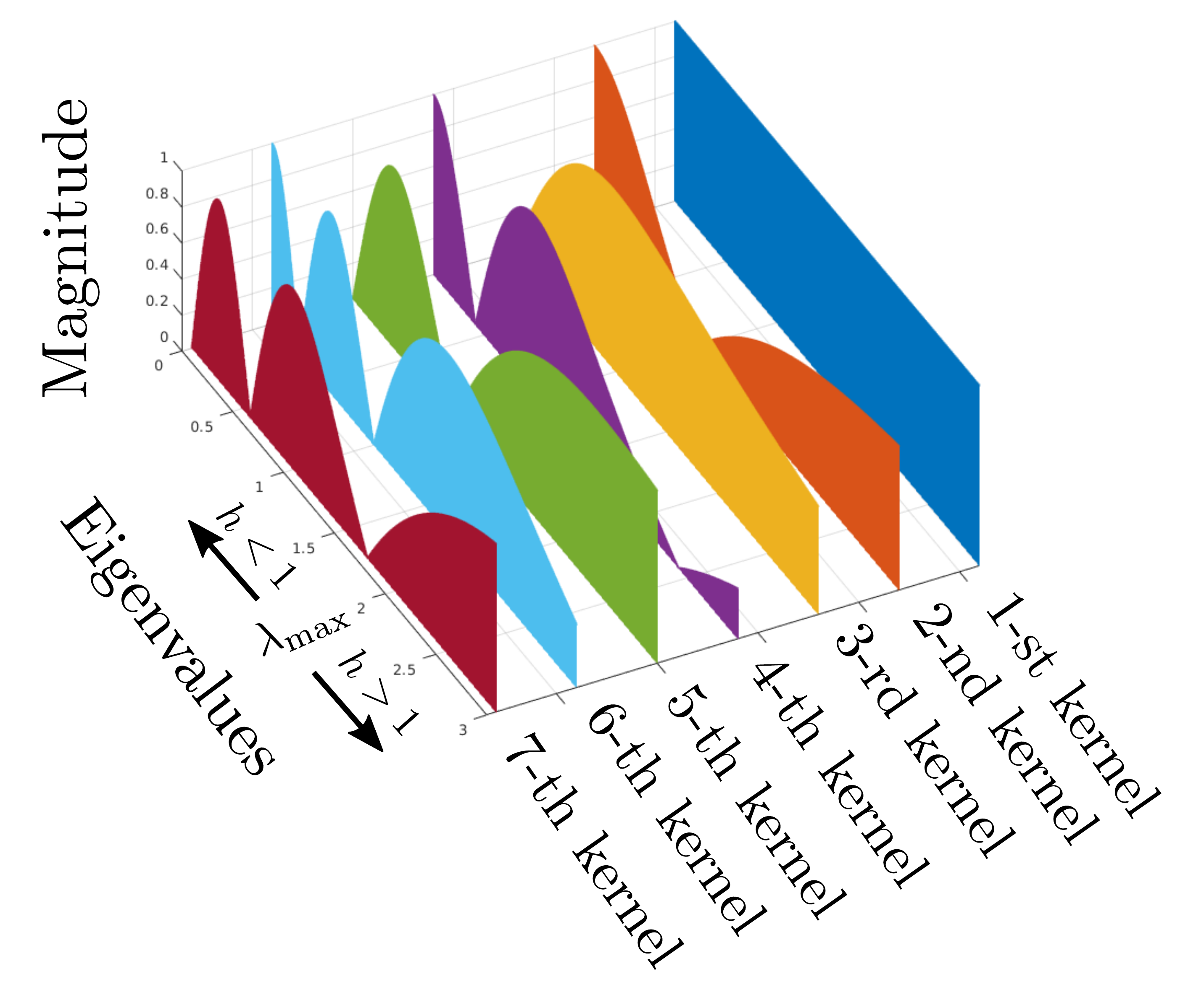}
\caption{Standard frequency profiles of first 7 CayleyNet convolutions. 
}
\label{fig:cayleyfreq}
\end{figure}
CayleyNet uses spectral graph convolutions whose frequency profiles can be changed by scaling eigenvalues [9]. The frequency profile is defined by a complex rational function of eigenvalues, scaled by a trainable parameter $h$ in \eqref{eq:Eqcy}. 

As proven in Appendix B, 
CayleyNet can be defined through the frequency profile matrix $B$. 
Using this representation, CayletNet can be seen as multi-kernel convolutions with real-valued trainable coefficients. According to this analysis, CayleyNet uses $2r+1$ graph convolution kernels, with $r$ being the number of complex coefficients [9]. The first 7 kernel's frequency profiles are illustrated in \figurename~\ref{fig:cayleyfreq}. The scale parameter $h$ affects the x-axis scaling but does not change the global shape. When $h=1$, frequency profiles can be defined within the range $[0,2]$ (because $\lambda_{\max}=2$ in all three test graphs). If $h=1.5$, the frequency profile can be defined till $1.5\lambda_{\max}=3$ in \figurename~\ref{fig:cayleyfreq} and rescale axis label from $[0,3]$ to $[0,2]$ in original range.

Learning the scaling of eigenvalues may seem advantageous. However, it induces extra computational cost in order to calculate the new convolution kernel. To limit this cost, an approximation is computed using a fixed number of Jacobi iterations~[9].
 
In addition, similarly to ChebNet, CayleyNet does not have any band specific convolutions, even when considering different scaling factors. 

\subsection*{GCN } 
\label{subsecGCN}

\begin{figure}[t]
  \centering
  \subfigure[Standard frequency profiles]{
    \includegraphics[width=.43\textwidth]{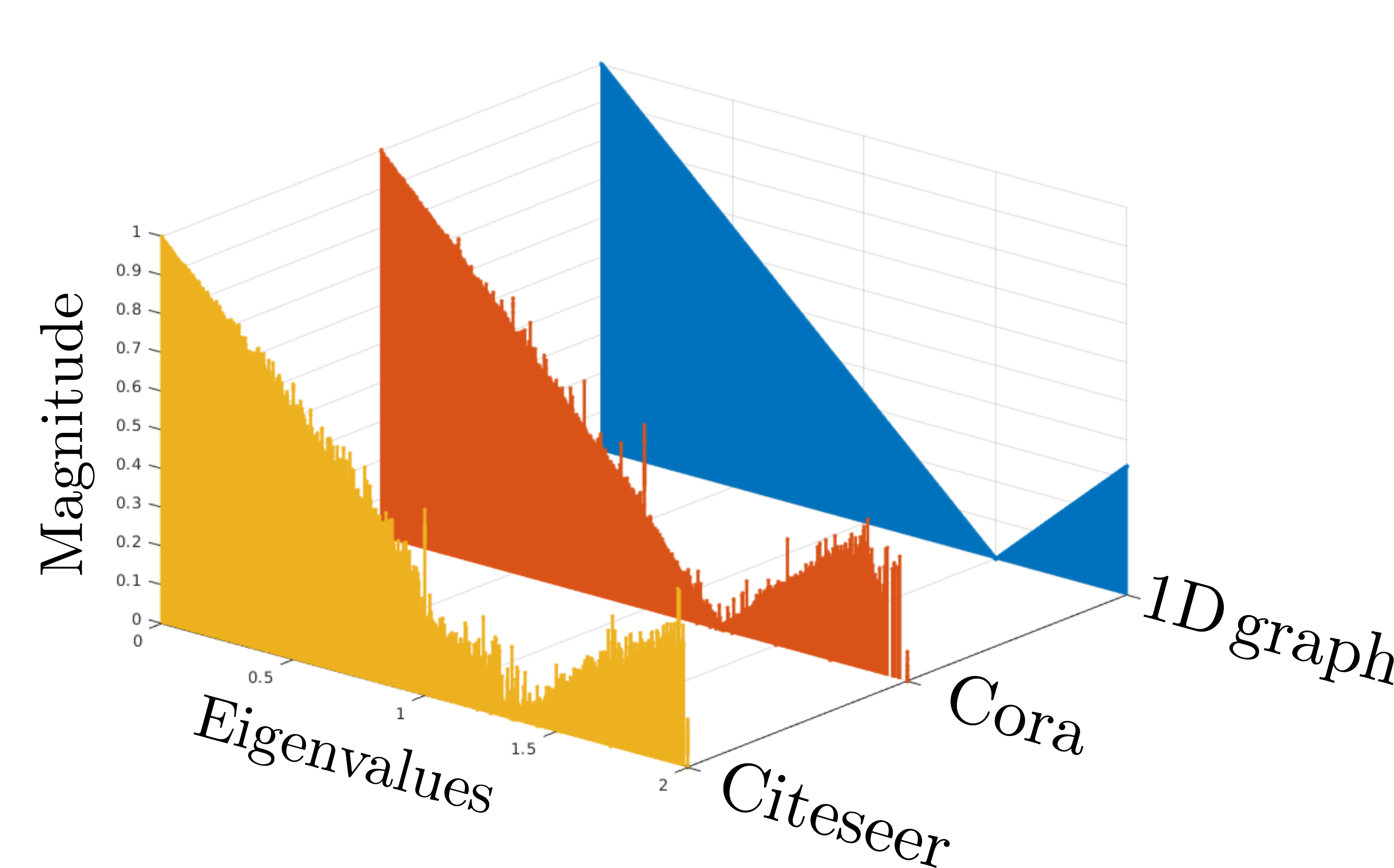}
  }
  \subfigure[Full frequency profile on  1D regular line graph]{
    \includegraphics[width=.43\textwidth]{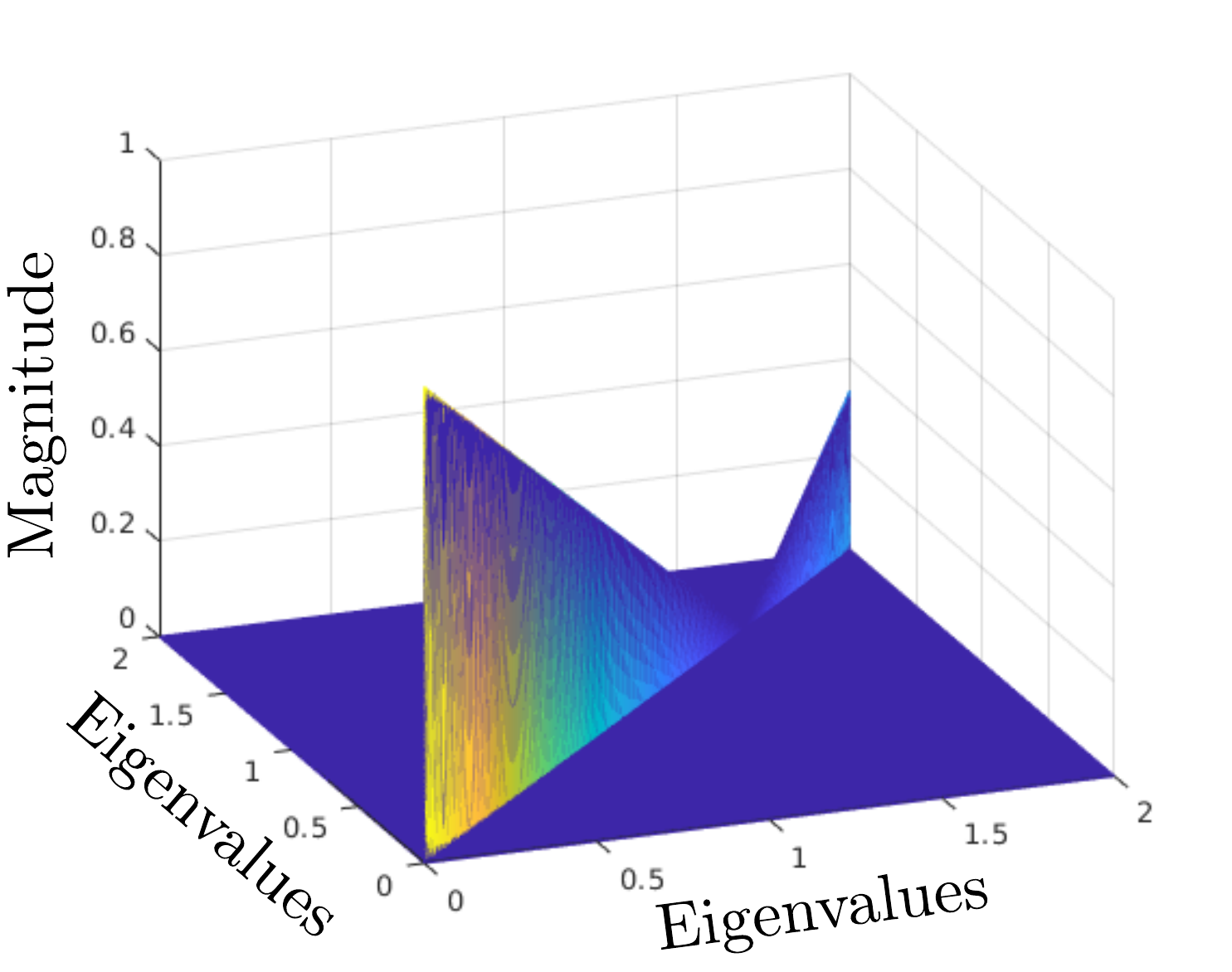}
  }
  \subfigure[Full frequency profile on Cora]{
    \includegraphics[width=.43\textwidth]{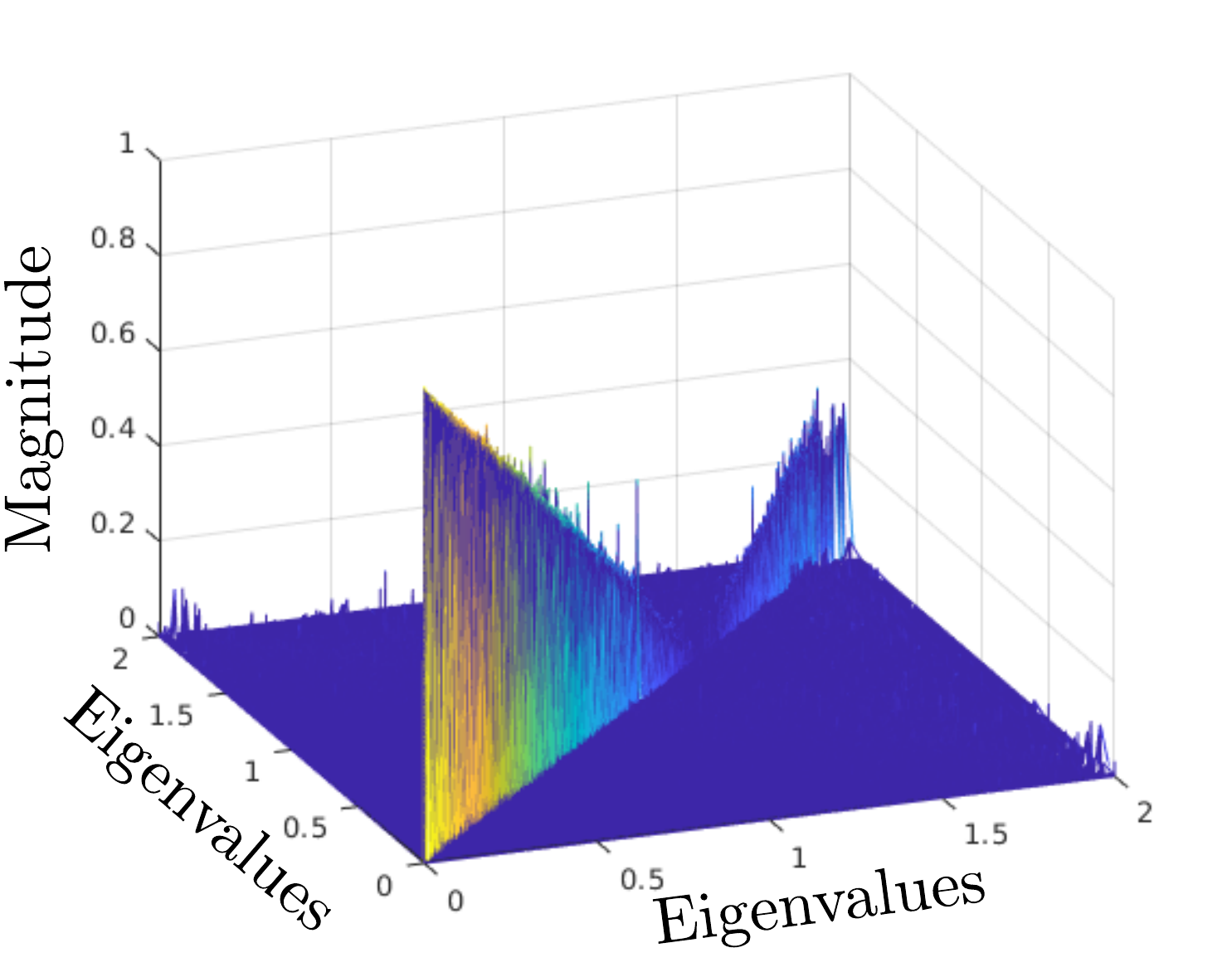}
  }
  \subfigure[Full frequency profile on Citeseer]{
    \includegraphics[width=.43\textwidth]{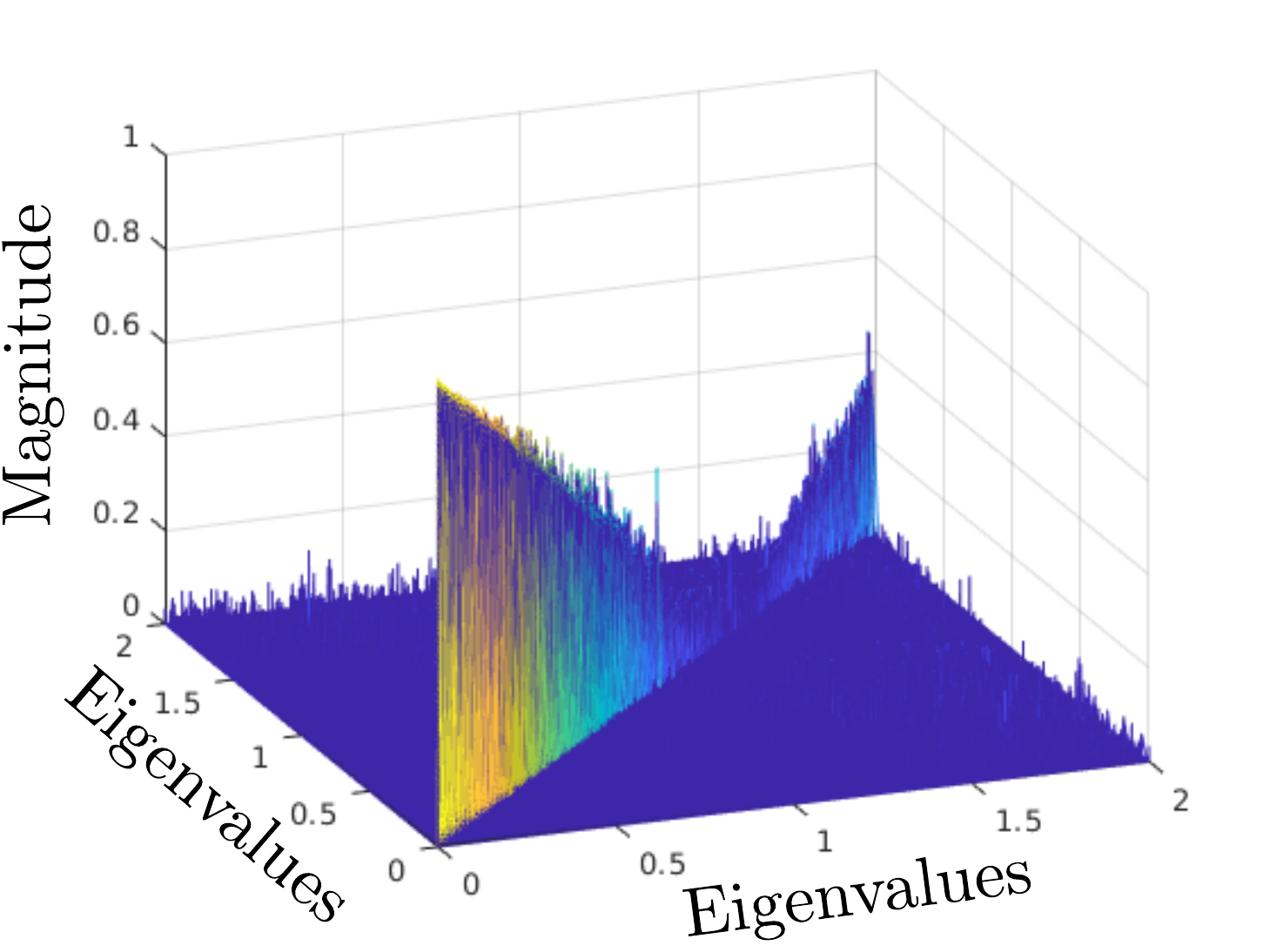}
  }
  \caption{Frequency profiles of GCN on different graphs.}
  \label{fig:gcnnfreq}
  \medskip
\end{figure}

As for ChebNet, a theoretical analysis of frequency profiles of GCN convolution is carried out in Appendix \ref{section:gcndetails}. It shows that GCN frequency profile can be approximated according to $\digamma~(\blambda) \approx {\bf1}-\blambda\overline{d}/(\overline{d}+1)$, where $\overline{d}$ is the average node degree. Therefore, the cut-off frequency of the GCN convolution is $\lambda_{\text{cut}}\approx(1+\overline{d})/\overline{d}$. Theoretically, if all  nodes degree are different, standard frequency profile will not be smooth and will include some perturbations. In addition, full frequency profile will be composed of non-zero components.  

Analyzing experimentally the behavior of GCN~[10] in the spectral domain first implies to compute the convolution kernel as given in~\eqref{eq:Eq14}. Then, the spectral representation of the obtained convolution matrix can be back-calculated using Corollary~\ref{cor:frequency_profile}. This result leads to the frequency profiles illustrated in \figurename~\ref{fig:gcnnfreq} for the three different graphs. The three standard frequency profiles have almost the same low-pass filter shape corresponding to a function composed of a decreasing part on the three first quarters of the eigenvalues range, followed by an increasing part on the remaining range. This observation is coherent with the theoretical analysis. Hence, kernels used in GCN are transferable across the three graphs at hand. In \figurename~\ref{fig:gcnnfreq}, the cut-off frequency of the 1-D linear circular graph is exactly 1.5, while it is about 1.35 for Citeseer. This observation can be explained by the fact that when considering a 1-D linear circular graph, all nodes have a degree equal to 2, hence  $\lambda_{\text{cut}}=
1.5$. Since the average node degree in Citeseer is 2.77, therefore $\lambda_{\text{cut}} \approx 
1.36$.

Concerning the full frequency profiles, there is no contribution outside the diagonal for the regular line graph (\figurename ~\ref{fig:gcnnfreq} b). Conversely, some off-diagonal values are not null for Citeseer and Cora. Again, this observation confirms the  theoretical analysis. 

Since GCN frequency profile does not cover the whole spectrum, such an approach is not able to learn relations that can be represented by high-pass or band-pass filtering. Hence, even though it gives very good results on a single graph node classification problem in [10], it may fail for problems where discriminant information lies in particular frequency bands. Therefore, such an approach can be considered as problem specific. 

\subsection*{GAT }
\begin{figure*}
  \centering

  \subfigure[Standard frequency profile] {
    \includegraphics[width=.31\textwidth]{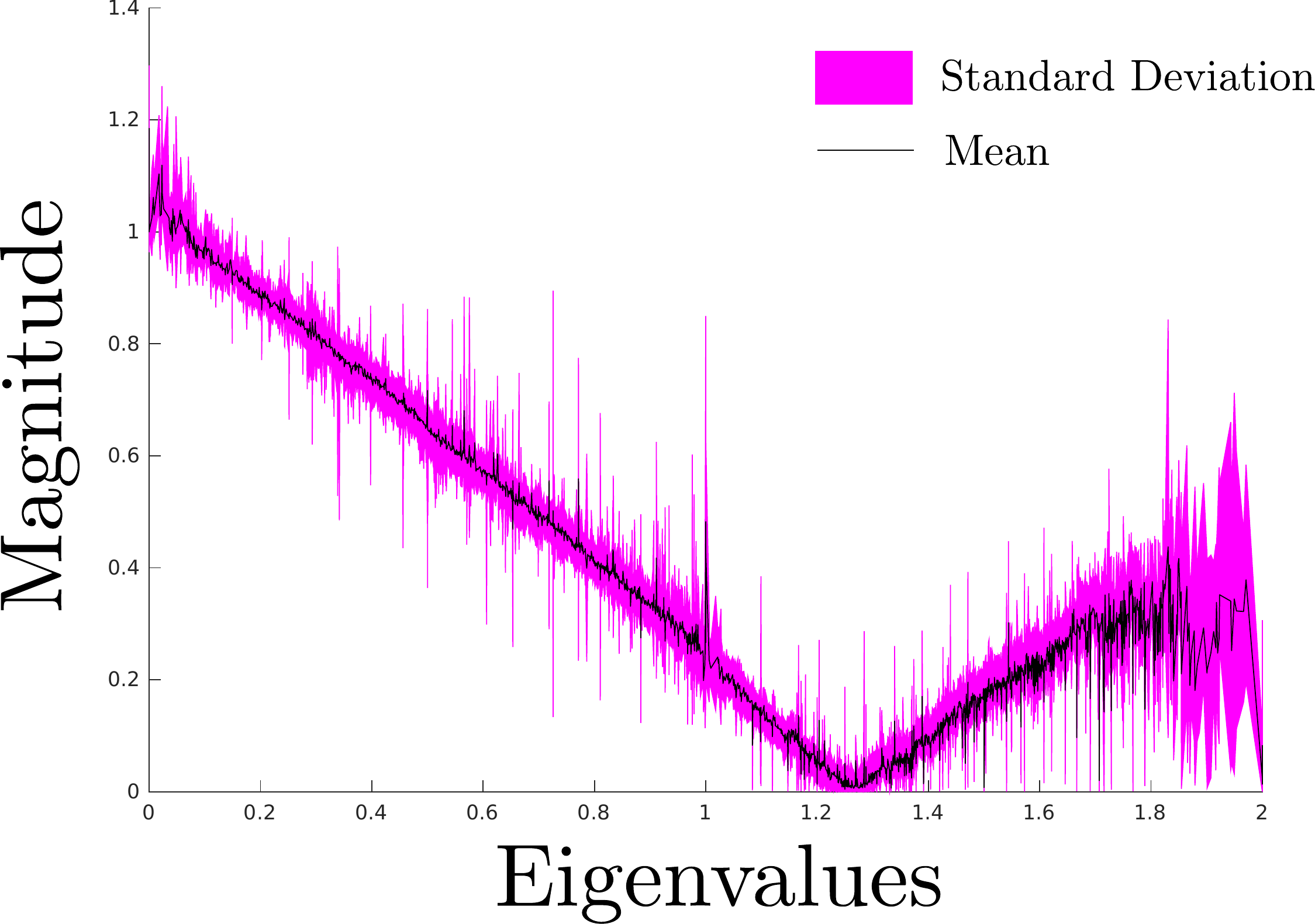}
  }
  \subfigure[Mean of full frequency profile] {
    \includegraphics[width=.31\textwidth]{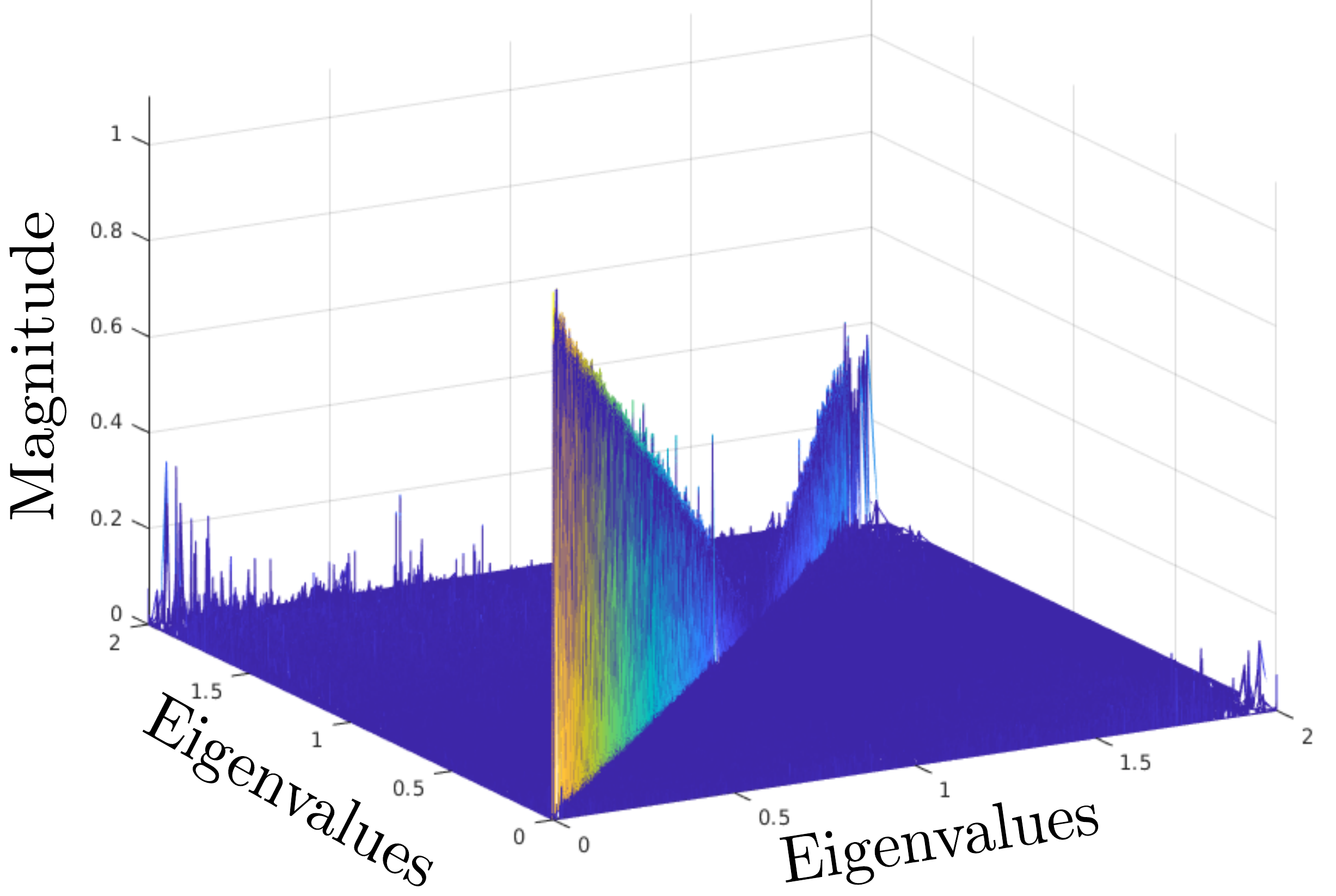}
  }
  \subfigure[Standard deviation of full frequency profile] {
   \includegraphics[width=.31\textwidth]{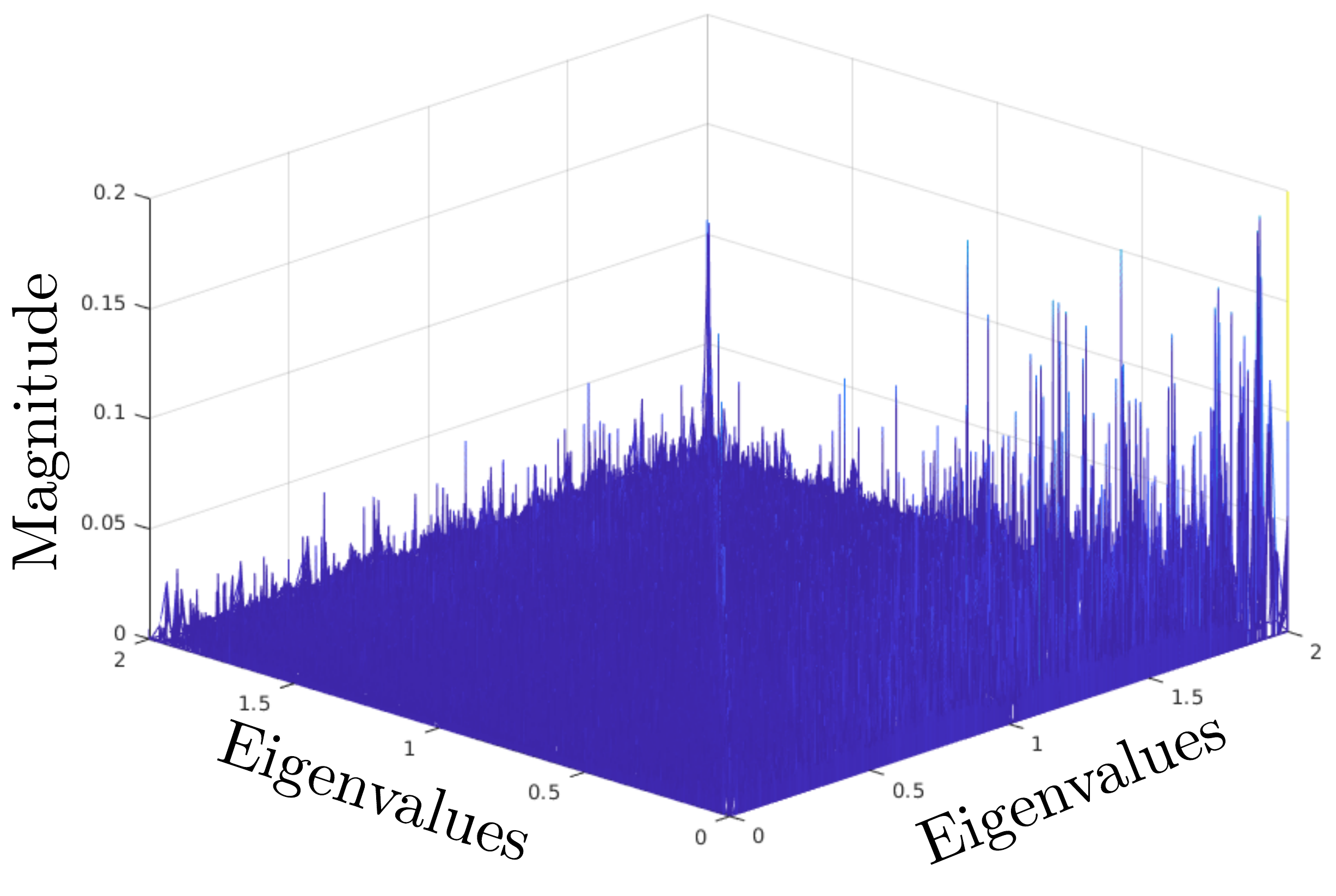}
  }
  
  \caption{Frequency profiles of randomly generated 250 GAT convolutions using Cora graph.}
\label{fig:gatfreqz}
\medskip
\end{figure*}

Graph attention networks (GATs) rely on trainable convolutions kernels~[11]. For this reason, frequency profiles cannot be directly computed similarly to GCN or ChebNet ones. Thus, instead of back-calculating the kernels, we perform simulations and evaluate the potential kernels of attention mechanism for given graphs. Hence, we show the frequency profiles of those simulated potential kernels. 

In ~[11], 8 different attention heads are used. Assuming that each attention head matrix is a convolution kernel, multi-attention systems can be seen as multi-kernel convolutions. The difference is that convolution kernels are not a priori defined but are functions of node feature vectors and trainable parameters $\mathbf{a}$ and $\mathbf{W}$; see \eqref{eq:gat2}. 
To show the potential output of GATs on the Cora graph (1433 features for each node), we produce 250 random pairs  of $\mathbf{W} \in R^{1433 \times 8}$ and $\mathbf{a} \in R^{16 \times 1}$, which correspond to the convolution kernels trained by GATs. The $\sigma$ function in \eqref{eq:gat2} is a $\verb|LeakyReLU|$ activation with a 0.2 negative slope as in ~[11].

The mean and standard deviation of the frequency profiles for these simulated GAT kernels are shown in \figurename~\ref{fig:gatfreqz}. As one can see, the mean standard frequency profile has a similar shape as those of GCN (\figurename~\ref{fig:gcnnfreq}). However, variations on the frequency profile induce more variations on output signal when compared to GCN.  

The full frequency profile is not symmetric. According to  \figurename~\ref{fig:gatfreqz}, variations are mostly on the right side of the diagonal in the full frequency profile. This is 
related to the fact that these convolution kernels are not symmetric. However, the variation on frequency profile might not be sufficient in problems that need some specific band-pass filters.

\subsection*{Discussion}

This section has shown that most influential graph convolutions~[10,11] operate as low-pass filters. Interestingly, while being restricted to low-pass filters, they still obtain state-of-the-art performance on particular node classification problems such as Cora and Citeseer~[12]. These results  on these particular problems are induced by the nature of the graphs to be processed. Indeed, citation network problems are inherently low-pass filtering problems, similarly to image segmentation problems, which are efficiently tackled by low-pass filtering. 

It is worth noting that, if we use enough convolution kernels, the frequency response of ChebNet kernels [8] covers nearly all frequency profiles. However, these frequency responses are not specific to special bands of frequency. It means that they can act as high-pass filters, but not as  Gabor-like special band-pass filters.

As a conclusion, we claim that graph convolutions presented in this section are problem specific and not problem agnostic. Experiments conducted in Section~\ref{sec:xp} provide empirical results to validate the theoretical analysis conducted in this section.

\begin{figure}[!t]
  \centering
  \includegraphics[width=.65\textwidth]{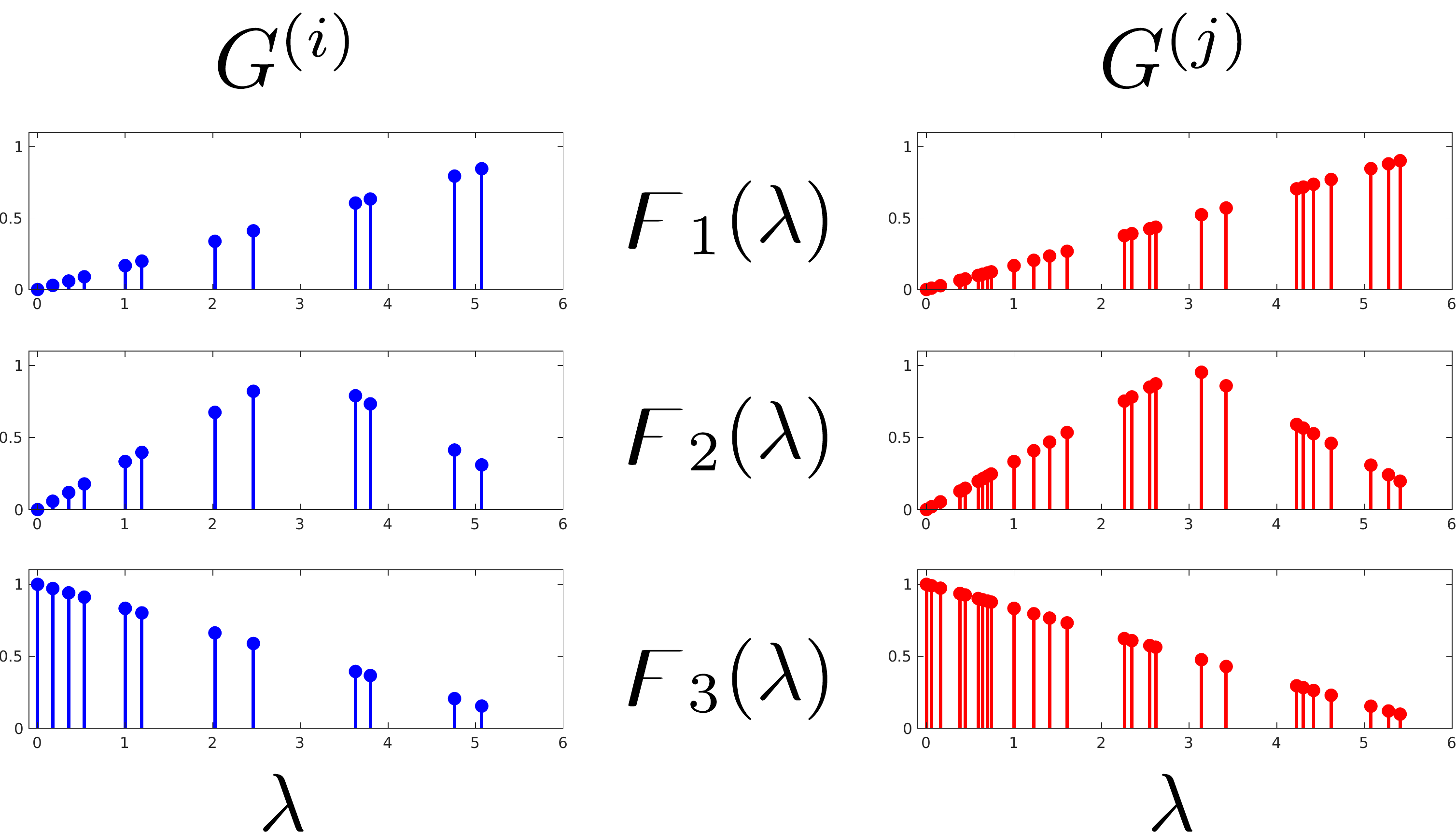}
  \caption{Three designed convolution kernel frequency profiles as a function of graph eigenvalues $(\blambda)$ of two sample graphs $G^{(i)}$ and $G^{(j)}$  by $\digamma_1(\blambda)=\frac{\blambda}{6},\digamma_2(\blambda)={\bf1} -\frac{| \blambda - 3|}{3}$ and $\digamma_3(\blambda)= {\bf1} -\frac{\blambda}{6}$. There are three shared coefficients. Each coefficient encodes the contribution of corresponding frequency profiles. First row refers mostly to high frequencies, middle row to middle frequencies and last row to low frequencies.}
  \label{fig:spec_design}
\end{figure}

\subsection{Depthwise Separable Graph Convolutions}
\label{section:depthwise}

Instead of designing the spatial convolution kernels $C^{(s)}$ of \eqref{eq:Eq3}  by  functions of graph adjacency and/or graph Laplacian, we propose in this section to use $S$ convolution kernels that have custom-designed standard frequency profiles. These designed frequency profiles are a function of eigenvalues, such as $\left[\digamma_1(\blambda), \dots, \digamma_S(\blambda)\right]$. In this proposal, the number of kernels and their frequency profiles 
are hyperparameters. 
Then, we can back-calculate corresponding spatial convolution matrices using \eqref{eq:Eq9a} in Theorem~\ref{Th:th1}.

To obtain problem-agnostic graph convolutions, the sum of all designed convolutions’ frequency profiles has to cover most of the possible spectrum and each kernel's frequency profile must focus on some certain ranges of frequencies. 
As a didactic example, we show in \figurename~\ref{fig:spec_design} an example of desired spectral convolutions frequency profiles for $S=3$ and its application on two different graphs.

In order to figure out arbitrary relations of input-output pairs, multiple convolution kernels have to be efficiently designed. However, increasing the number $S$ of convolution kernels  increases the number of trainable parameters linearly. Hence,  the total number of multi-support ConvGNN  is given by $S\sum_{i = 0}^L f_i f_{i+1} $ where $L$ is the number of layers and $f_i$ is the feature length of the $i$-th layer.

To overcome this issue, we propose to use Depthwise Separable Graph Convolution Network (DSGCN). Depthwise Separable Convolution framework has already been used in computer vision problems to reduce
the model size and its complexity~[30,31]. To the best of our knowledge, depthwise separable graph convolution has never been proposed  in the literature.  

 Instead of filtering all input features for each output feature, DSGCN consists in  filtering each input feature once. Then, filtered signals are merged into the desired number of output features through 1$\times$1 convolutions with  different contribution coefficients. Detailed illustration of the proposed depthwise separable graph convolution process is presented in \figurename~\ref{fig:depthwise_layer}.

\begin{figure}[t]
\centering
 \includegraphics[width=.95\textwidth]{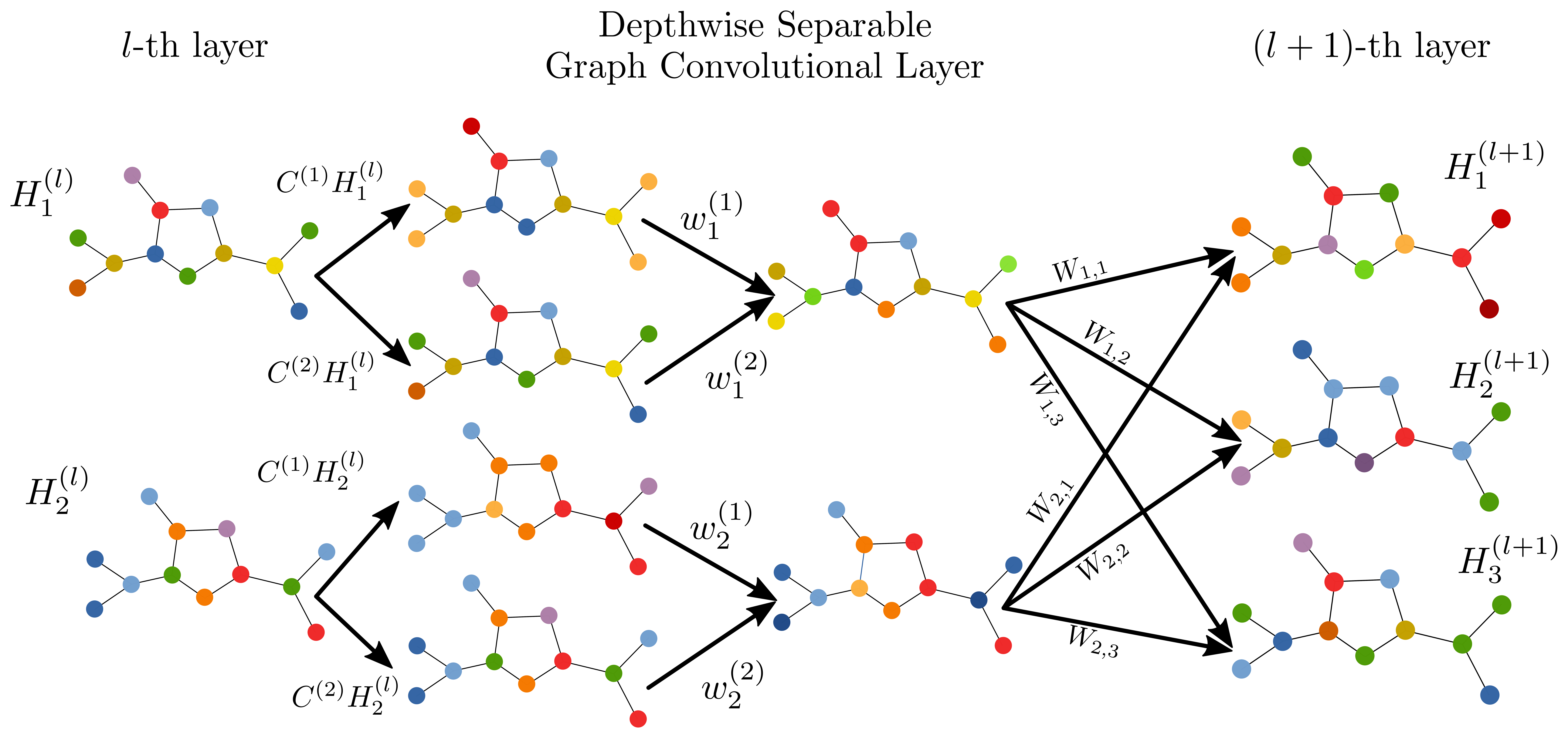}
 \caption{Detailed schematic of Depthwise Separable Graph Convolution Layer. Each node has a 2-length feature vector, indicated as $H_1^{(l)}$ and $H_2^{(l)}$ with values represented by colors. The following layer has a 3-length feature vector, denoted $H_1^{(l+1)}$, $H_2^{(l+1)}$ and $H_3^{(l+1)}$. Here, two convolution kernels are used, denoted by $C^{(1)}$ and $C^{(2)}$. Convoluted signals are multiplied by trainable weight $w$ and are summed to obtain interlayer signals.  To obtain the 3 next layer features, a weighted sum is computed using the other trainable parameter $W$.
  }
 \label{fig:depthwise_layer}
 \medskip
\end{figure}

Mathematically, forward calculation of  each layer of DSGCN is defined by: 
\begin{equation}
  \label{eq:EqDepth}
  H^{(l+1)} = \sigma \left( \Big(\sum_{s=1}^S w^{(s,l)} \odot
  (C^{(s)}H^{(l)}) \Big)  W^{(l)}\right).
\end{equation}
In this expression, the notation $\odot$  denotes the element-wise multiplication operator. Note that there is only one trainable matrix $W$ in each layer. Other trainable variables $w^{(s,l)} \in \mathbb{R}^{1 \times f_l}$
encode feature contributions for each convolution kernel 
and layer. 
The number of trainable parameters for this case becomes $\sum_{i=0}^L S f_i + f_i f_{i+1}$. Previously, adding a new kernel  increases the number of parameters by $\sum_{i=0}^L f_i f_{i+1}$. Using separable convolutions, this number is only increased by
$\sum_{i=0}^L f_i$. This modification is particularly interesting when the number of features is high. On the other hand, the variability of the model also decreases. If the data has a smaller number of features, using this approach might not be optimal.

\section{Experimental evaluation}
\label{sec:xp}

In this section, we describe the experiments carried out to evaluate the proposed approach on both transductive and inductive problems. In the first case, we target a single graph node classification task while in the second case, both multi-graph node classification task and entire graph classification task are considered (see Section \ref{secGLP}). For all the experiments, we compare our algorithm to state-of-the-art approaches.

\subsection{Transductive Learning Problem}

\subsubsection{Datasets} 

Experiments on transductive problems were led on the three datasets summarized in \tablename~\ref{tabletrsum}. These datasets are well-known paper citation graphs. Each node corresponds to a paper. If one paper cites another one, there is an unlabeled and undirected edge between the corresponding nodes. Binary features on the nodes indicate the presence of specific keywords in the corresponding paper. The task is to attribute a class to each node (i.e., paper) of the graph using for training the graph itself and a very limited number of labeled nodes. Labeled data ratio is 5.1\%, 3.6\% and 0.3\% for Cora, Citeseer and PubMed respectively. We use predefined train, validation and test sets as defined in [12] and follow the test procedure of [10,11] for fair comparisons. 

\begin{table}[!t]
\renewcommand{\arraystretch}{1.3}
\caption{Summary of the transductive datasets used in our experiments. \newline Each dataset consists of one single graph}
\centering

\begin{tabular}{l c c c}
  & Cora & Citeseer & PubMed \\
\toprule
\# Nodes  &2708  & 3327   &19717 \\
\# Edges &5429 &4732 &44338 \\
\# Features &1433 &3703 &500 \\
\# Classes &7 &6 &3 \\
\# Training Nodes &140 &120 &60 \\
\# Validation Nodes &500 &500 &500 \\
\# Test Nodes &1000 &1000 &1000 \\
\bottomrule
\end{tabular}
\label{tabletrsum}
\end{table}

\subsubsection{Models} 

To evaluate the performance of 
convolutions designed in the spectral domain 
independently from the architecture design, a single hidden layer is used for all models, as in [10] for GCN. This choice, even sub-optimal, enables 
a deep understanding of the 
convolution kernels. For these evaluations, a set of convolution kernels is experimented:
\begin{itemize}
    \item A low-pass filter defined by $\digamma_{1}(\blambda)=({\bf1}-\blambda/\lambda_{\max})^\eta$ where $\eta$ impacts the cut-off frequency 
    \item A high-pass filter defined by $\digamma_{2}(\blambda)=\blambda/\lambda_{\max}$ 
    \item Three band-pass filters defined by: \begin{itemize}
        \item $\digamma_{3}(\blambda)=\exp(-\gamma(0.25\lambda_{\max}-\blambda)^2)$
        \item $\digamma_{4}(\blambda)=\exp(-\gamma(0.5\lambda_{\max}-\blambda)^2)$
        \item $\digamma_{5}(\blambda)=\exp(-\gamma(0.75\lambda_{\max}-\blambda)^2)$
    \end{itemize}
    \item An all-pass filter defined by $\digamma_{6}(\blambda)={\bf1}$
\end{itemize}

We firstly consider a model composed of only $\digamma_{1}$. This choice comes from the fact that state-of-the-art GNNs are sort of low-pass filters (see Section \ref{secSAGC}) and perform well on the datasets of \tablename~\ref{tabletrsum}. Hence, it is interesting to evaluate our framework with $\digamma_{1}$. For the experiments, the value of $\eta$ are tuned for each dataset, using the validation loss value and accuracy, yielding $\eta=5$ for Cora and Citeseer, and $\eta=3$ for PubMed. Details concerning this tuning can be found in \tablename~A1 in Appendix~D. Since there is only one convolution kernel, depthwise separable convolutions are not necessary for this model. Therefore, this model can be seen as similar to those from  [8,10] but using a different convolution kernel. This approach is denoted as {LowPassConv} in the results section (Section \ref{sec:results}).

Beyond this low-pass model, we also evaluate different combinations of the $\digamma_i(\blambda)$ through the depthwise separable schema defined in Section~\ref{section:depthwise}. For experiments involving  $\{\digamma_{3}(\blambda), \digamma_{4}(\blambda), \digamma_{5}(\blambda)\}$, the bandwidth parameter $\gamma$ was tuned using train and validation sets. \tablename~\ref{tablearch2} details the best models found on the validation set. As an example, for Cora dataset, 4 kernels are used by a DSGCN with 160 neurons: $\digamma_1(\blambda)$, $\digamma_3(\blambda)$, $\digamma_4(\blambda)$, $\digamma_5(\blambda)$. As an illustration, \figurename~\ref{ourcorafreq} provides the standard frequency profiles of this designed convolution on Cora dataset. The models of \tablename~\ref{tablearch2} are denoted as {DSGCN} in the following. 

\begin{table}
\renewcommand{\arraystretch}{1.3}
\caption{Used kernels frequency profiles and architecture of models for each transductive dataset. DSG refers to Depthwise Separable Graph convolution layer, G to Graph convolution layer, D to Dense layer}
\centering

\begin{tabular}{l c  }
\hline
Dataset & Architecture \\
\hline 
&  $\digamma_{1}(\blambda)=({\bf1}-\blambda/\lambda_{\max})^5$\\ 
& $\digamma_{3}(\blambda)=\exp(-0.25(0.25\lambda_{\max}-\blambda)^2)$ \\
Cora & $\digamma_{4}(\blambda)=\exp(-0.25(0.5\lambda_{\max}-\blambda)^2)$\\ 
& $\digamma_{5}(\blambda)=\exp(-0.25(0.75\lambda_{\max}-\blambda)^2)$  \\
& DSG160-DSG7 \\
\hline
Citeseer& $\digamma_{1}(\blambda)=({\bf1}-\blambda/\lambda_{\max})^5$,~$\digamma_{6}(\blambda)={\bf1}$ \\
 &  DSG160-DSG6 \\
\hline
Pubmed & $\digamma_{1}(\blambda)=({\bf1}-\blambda/\lambda_{\max})^3$,~$\digamma_{2}(\blambda)=\blambda/\lambda_{\max}$ \\
 & DSG16-DSG3 \\
 \hline
\end{tabular}
\label{tablearch2}
\end{table}

\begin{figure}[!t]
\centering
\includegraphics[width=.6\textwidth]{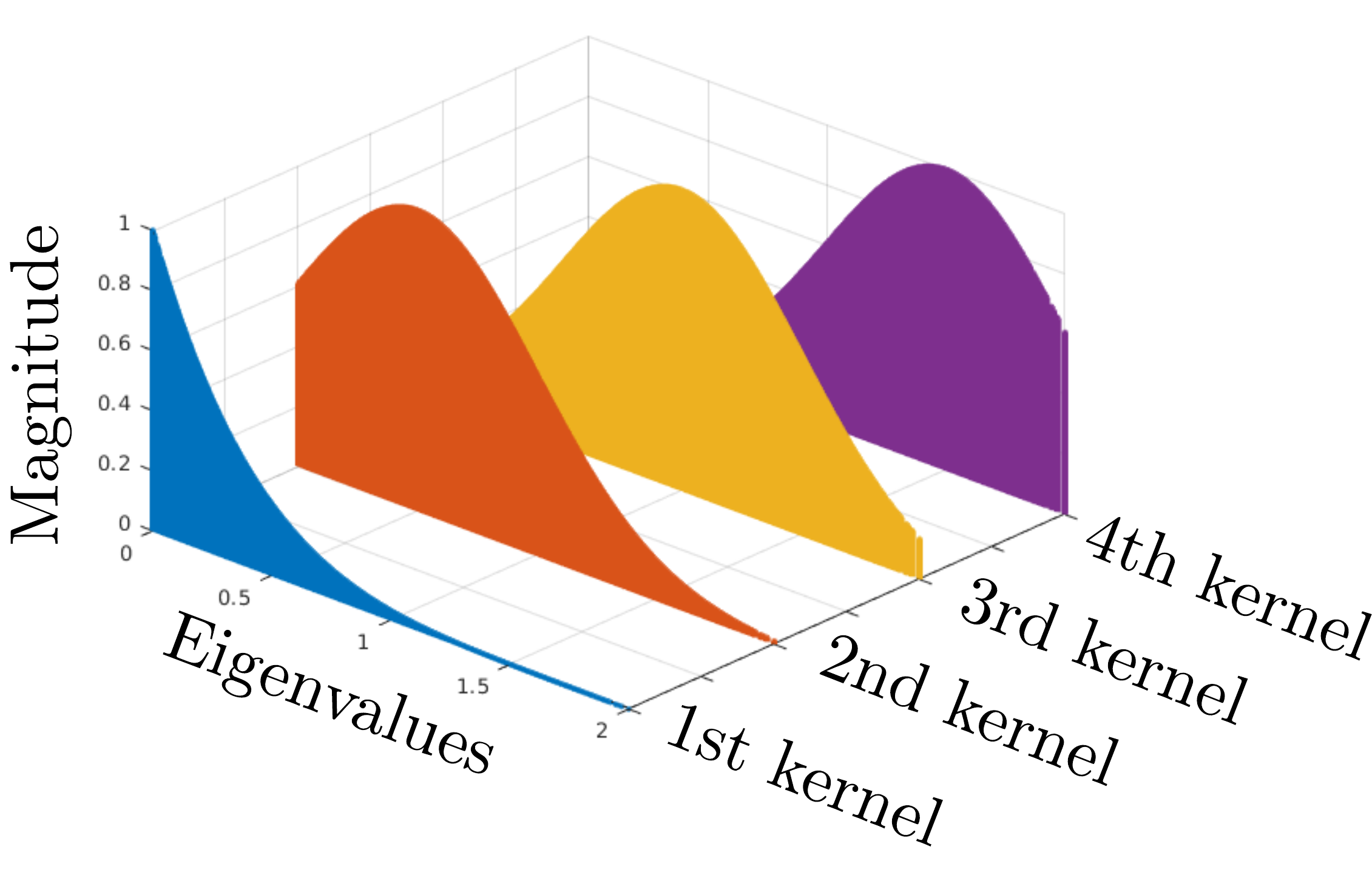}
\caption{Designed convolution's frequency profiles for Cora dataset.}
\label{ourcorafreq}
\end{figure}

The training hyperparameters were tuned over a grid search using a cross-validation procedure
. Hyperparameter values can be found in \tablename~A2 of Appendix D. Other protocol details are also given in this appendix.

\subsubsection{Results}
\label{sec:results}

Obtained results on transductive learning are given in \tablename~\ref{tab:trans-first}. We compare the performance of the proposed {LowPassConv} and {DSGCN} to state-of-the-art methods.

We first can see that our low-pass convolution kernel (LowPassConv)  obtains comparative performance with  existing methods. This result confirms our theoretical analysis which states that GCN and GAT mostly correspond to low-pass filters (Section~\ref{secSAGC}). Second, DSGCN outperforms state-of-the-art methods thanks to the flexibility  provided by the different filters. 
It is worth noting that the good results obtained by low-pass approaches show that these three classification tasks are mainly low-pass specific problems. Differences in accuracies may be significantly bigger for band-pass or high-pass based problems.

\begin{table}[!t]
\renewcommand{\arraystretch}{1.3}
\caption{Comparison of methods on the transductive learning problems  using publicly defined train, validation and test sets. Accuracies on test set are reported with their standard deviations under 20 random runs.}

\centering
\begin{tabular}{l c c c }

 Method &  Cora & Citeseer & Pubmed \\
\toprule

MLP &0.551     &0.465   & 0.714  \\
Planetoid [12]
&0.757    &0.647    & 0.744  \\
MoNet [23]  &0.817 $\pm$ 0.005   &-  & 0.788 $\pm$ 0.003  \\
ChebNet [8]  &0.812    &0.698   & 0.744  \\
CayleyNet [9] &0.819 $\pm$ 0.007   &-   & -  \\
DPGCNN [26] &0.833 $\pm$ 0.005 & 0.726 $\pm$ 0.008  &- \\
GCN [10]  &0.819 $\pm$ 0.005    &0.707 $\pm$ 0.004    & 0.789 $\pm$    0.003    \\
GAT [11]  &0.830 $\pm$ 0.007   &0.725 $\pm$ 0.007  & 0.790 $\pm$ 0.007   \\

\midrule
\textbf{LowPassConv} &0.827 $\pm$   0.006	& 0.717 $\pm$    0.005 &0.794  $\pm$ 0.005 \\

\textbf{DSGCN} &  \textbf{0.842  $\pm$  0.005}	& \textbf{0.733 $\pm$   0.008}	& \textbf{0.819  $\pm$ 0.003}	 \\
\bottomrule

\end{tabular}
\label{tab:trans-first}

\end{table}

\begin{table}[!t]
\renewcommand{\arraystretch}{1.3}
\caption{Summary of inductive learning datasets used in this paper.}
\centering

\begin{tabular}{l c c c}
\toprule
  & PPI & PROTEINS & ENZYMES\\
\midrule
\ Type & Node Class. & Graph Class. & Graph Class. \\
\# Graph  &24  &1113  &600  \\
\# Avg.Nodes  &2360.8  &39.06  &32.63 \\
\# Avg.Edges &33584.4 &72.82 & 62.14 \\
\# Features &50 &3 label & 3 label + 18 cont. \\
\# Classes &2 (121 criterias) &2   &6  \\
\# Training &20 graphs & 9-fold & 9-fold  \\
\# Validation &2 graphs & 1-fold   & 1-fold  \\
\# Test &2 graphs &  None & None  \\
\bottomrule
\end{tabular}
\label{tableindsum}
\end{table}

\subsection{Inductive Learning Problem}
Inductive Learning problems are common in chemoinformatics and bioinformatics. In an inductive setting, a given instance is represented by a single graph. Thus, models are trained and tested on different graph sets.

In the graph neural networks literature, there is a controversy concerning the transferability of spectral designed convolutions from learning graphs to unseen graphs. Some authors consider that convolutions cannot be transferred [23], while very recent theoretical  [32] and empirical [33] works show the contrary. 
In this subsection, we target to bring an answer to this controversy by experimenting our proposal on inductive learning problems. 

\subsubsection{Datasets} 

Inductive experiments are led on 3 datasets (see \tablename~\ref{tableindsum} for a summary): a multi-graph node classification dataset called Protein-to-Protein Interaction (PPI) [34] and on two graph classification datasets called PROTEINS and ENZYMES [35]. The protocols used for the evaluations are those defined in [11] for PPI and  [36,37,38,39] for PROTEINS and ENZYMES datasets.

The PPI dataset is a multi-label node classification problem on multi-graphs. Each node has to be classified either True or False for 121 different criteria. All the nodes are described by a 50-length continuous feature vector. The PPI dataset includes 24 graphs, with a train/validation/test standard splitting. 

The PROTEINS and ENZYMES datasets are graph classification datasets. There are 2 classes in PROTEINS and 6 classes in ENZYMES. In PROTEINS dataset, there are three different types of nodes and one continuous feature. But we do not use this continuous feature on nodes. In ENZYMES dataset, there are 18 continuous node features and three different kinds of node types. In the literature, some methods use all provided continuous node features while others use only node label. This is why ENZYMES results are given using either all features (denoted by ENZYMES-allfeat) or only node labels (denoted by ENZYMES-label). 

Since there is no standard train, validation and test sets split for PROTEINS and ENZYMES, the results are given using a 10-fold cross-validation (CV) strategy under a fixed predefined epoch number. The CV only uses training and validation set. Specifically, after obtaining 10 validation curves corresponding to 10 folds, we first take average of validation curves across the 10 folds and then select the single epoch that achieved the maximum averaged validation accuracy. This procedure is repeated 20 times with random seeds and random division of dataset. Mean accuracy and standard deviation are reported. This is the same protocol than [36,37,38,39].

\subsubsection{Models} 

For PPI, 7 depthwise graph convolution layers compose the model. Each layer has 800 neurons, except the output layer which has 121 neurons, each one classifying the node either True or False. All layers use a ReLU activation except the output layer, which is linear. No dropout or regularization of the binary cross-entropy loss function is used. All graph convolutions use three spectral designed convolutions: a low-pass convolution given by $\digamma_{1}(\blambda)=\exp(-\blambda/10)$, a high-pass one given by $\digamma_{2}(\blambda)=\blambda/\lambda_{\max}$ and an all-pass filter given by $\digamma_{3}(\blambda)={\bf1}$.

For graph classification problems (PROTEINS 
and ENZYMES
), depthwise graph convolution layers are not needed since these datasets have a reduced number of features. Thus, it is tractable to use all multi-support graph convolution layers instead of the depthwise schema. In these cases, our models firstly consist of a series of graph convolution layers. Then, a global pooling (i.e., graph readout) is applied in order to aggregate extracted features at graph level. For this pooling, we use a concatenation of mean and max global pooling operator, as used in [37]. Finally, a dense layer (except for ENZYMES-label) is applied, before the output layer as in~[39]. 

All details about the architecture and designed convolutions can be found in \tablename~\ref{tablearch3}. The hyperparameters used in best models can be found on \tablename~\ref{tablehyper} in Appendix~\ref{section:appdetails}.

\begin{table}
\renewcommand{\arraystretch}{1.3}
\caption{Kernels frequency profiles and model architecture for each inductive  dataset. meanmax refers to global mean and max pooling layer. \newline Same legend as \tablename~\ref{tablearch2}.
}
\centering

\begin{tabular}{l c  }
\hline
Dataset & Architecture\\
\hline 
&  $\digamma_{1}(\blambda)=\exp(-\blambda/10)$ \\
PPI & ~$\digamma_{2}(\blambda)=\blambda/\lambda_{\max}$,~$\digamma_{3}(\blambda)={\bf1}$ \\ 
   & DSG800-DSG800-DSG800-DSG800-\\
   & DSG800-DSG800-DSG121\\
\hline

PROTEINS & $\digamma_{1}(\blambda)={\bf1}-\blambda/\lambda_{\max}$,~$\digamma_{2}(\blambda)=\blambda/\lambda_{\max}$ \\
   & G200-G200-meanmax-D100-D2\\
\hline

& $\digamma_{1}(\blambda)={\bf1}$,~$\digamma_{2}(\blambda)=\blambda_{s}-{\bf1}$\\
ENZYMES-label  & $\digamma_{3}(\blambda)=2\blambda_{s}^2-4\blambda_{s}+{\bf1}$,~$\blambda_{s}=2\blambda/\lambda_{\max}$\\ 
& G200-G200-G200-G200-meanmax-D6\\
\hline

& $\digamma_{1}(\blambda)={\bf1}$,~$\digamma_{2}(\blambda)=\exp(-\blambda^2)$ \\
ENZYMES-allfeat & $\digamma_{3}(\blambda)=\exp(-(\blambda-0.5\lambda_{\max})^2)$\\
& $\digamma_{4}(\blambda)=\exp(-(\blambda-\lambda_{\max})^2) $ \\
& G200-G200-meanmax-D100-D6\\
\hline

\end{tabular}
\label{tablearch3}
\medskip
\end{table}

\subsubsection{Results} 

\tablename~\ref{tableindres} compares the results obtained by the models described above and state-of-the-art methods. A comparison with the same models but without graph information, a Multi-Layer Perceptron (MLP) that corresponds to $C^{(1)}=I$ is also provided to discuss if structural data include information or not. To the best of our knowledge, such an analysis is not provided in the literature. Finally, results obtained by the same architecture with GCN kernel is also provided.

\begin{table}
\renewcommand{\arraystretch}{1.3}
\caption{Comparison of methods on inductive learning problems using publicly defined data split for PPI dataset and 10-fold CV for PROTEINS and ENZYMES datasets. PPI results are the test set results reported by micro-F1 metric percentage. Others are CV results reported by accuracy percentage. Results denoted by $^{*}$ were reproduced from original source codes but denoted feature set.}
\centering

\begin{tabular}{@{}l @{~}c c c c @{}}
\toprule
Method &  PPI & PROTEINS & \multicolumn{2}{c}{ENZYMES} \\
 & All Features & Node Label & Node Label & All Features \\
\midrule

GraphSAGE [22]  &76.8   & - &-&-
 \\

GAT [11]  
&97.3  $\pm$ 0.20  & - &- &-
 \\ 
GaAN [40]  
&98.7  $\pm$ 0.20  & - &- &-
 \\ 

Hierarchical Pooling
[37]
&   -
& 75.46 & 64.17 & - \\ 

Diffpool [36]
&   - 
& 76.30  & 62.50 & 66.66$^{*}$ \\ 

ChebNet 
[33]
&   -
&75.50 $\pm$ 0.40 &58.00 $\pm$ 1.40 &-  \\

Multigraph [33]
&  - 
&76.50 $\pm$ 0.40  &61.70 $\pm$ 1.30 &68.00 $\pm$ 0.83 \\

GIN [39] & - & 76.20 $\pm$ 0.86 & - & - \\
GFN [38]
&   -
& 76.56 $\pm$ 0.30$^{*}$  & 60.23 $\pm$ 0.92$^{*}$  & 70.17 $\pm$ 0.86  \\

\midrule

MLP ($C^{(1)}=I$) 
& 46.2 $\pm$ 0.56  
& 74.03 $\pm$ 0.92
& 27.83 $\pm$ 2.51 
& 76.11 $\pm$ 0.87   
\\

GCN \eqref{eq:Eq14}    
& 59.2 $\pm$ 0.52  
& 75.12 $\pm$ 0.82
& 51.33 $\pm$ 1.23
& 75.16 $\pm$ 0.65
\\

\textbf{DSGCN} 
&\textbf{99.09 $\pm$ 0.03}   
&\textbf{77.28 $\pm$ 0.38}
&\textbf{65.13 $\pm$ 0.65}
&\textbf{78.39 $\pm$ 0.63} \\ 
\bottomrule
\end{tabular}
\label{tableindres}
\medskip
\end{table}

As one can see in \tablename~\ref{tableindres}, the proposed method obtains competitive results on inductive datasets. For PPI, DSGCN clearly outperforms state-of-the-art methods with the same protocol, reaching a micro-F1 percentage of 99.09 and an accuracy of 99.45\%. For this dataset, MLP accuracy is low since the percentage of micro-F1 is 46.2 (random classifier's micro-F1 being 39.6\%). This means that the problem includes significant structural information. Using the GCN kernel, which operates as low-pass convolution (see Section \ref{subsecGCN}), the accuracy increases to 0.592, but again not comparable with state-of-the-art accuracy.

For the PROTEINS dataset, one can see that MLP ($C^{(1)}=I$) reaches an accuracy that is quite comparable with state-of-the-art GNN methods. Hence, MLP reaches a 74.03\% validation accuracy while the proposed DSGCN reaches 77.28\%, which is the best performance among GNNs. This means that PROTEINS problem includes very few structural information to be exploited by GNNs.

ENZYMES dataset results are very interesting in order to understand the importance of continuous features and their processing through different convolutions. As one can see in \tablename~\ref{tableindres}, there are important differences of performance between the results on ENZYMES-label and ENZYMES-allfeat. When node labels are used alone, without features, MLP accuracy is very poor and nearly acts as a random classifier. When using all features, MLP outperforms GCN and even some state-of-the-art methods. A first explanation is that methods are generally optimized for just node label but not for continuous features. Another one is that the continuous features already include information related to the graph structure since they are experimentally measured. Hence, their values are characteristic of the node when it is included in the given graph. Since GCN is just a low-pass filter, it removes some important information on higher frequency and decreases the accuracy. Thanks to the multiple convolutions proposed in this paper, our GNN DSGCN clearly outperforms other methods on the ENZYMES dataset.

\section{Conclusion}

The success of convolutions in neural network strongly depends on the capability of defined convolution kernels on producing outputs as different as possible. While this has been widely investigated for CNNs, there has not been any study for ConvGNNs with graph convolution, to the best of our knowledge. This paper proposed to fill this gap, by examining the graph convolutions as custom frequency profiles and taking  
advantage of using optimized multi-frequency profile convolutions. By this way, we significantly increased the performance on reference datasets. 

Nevertheless, the proposed approach has some drawbacks. First, it needs eigenvalues and eigenvectors of the graph Laplacian. If the graph has more than 20k nodes, computing these values is not tractable. Second, we did not propose yet any automatic procedure to select the best frequency profile of convolution. Hence, the proposed approach needs expertise to find the appropriate graph kernels. Third, although our theoretic complexity is the same than GCN or ChebNet, in practice our convolutions are more dense than GCN, which makes it slower in practice since it cannot take advantage of sparse matrix multiplications. Last, if edge type can be handled by designing convolution for each type, the proposed method does not handle continuous edge features and directed edges. 

Our future work will target the automatic design of graph convolutions in spectral domain. It may be done by unsupervised manner as preprocessing step. Another future work will be on handling given continuous edge features and directed edge in our framework. Also, we have a plan to design convolution frequencies not by function of eigenvalues but through linear combination of Chebyshev kernels in order to skip the necessity of eigenvalue calculations. 

\section*{Acknowledgment}

This work was partially supported by the ANR grant APi (ANR-18-CE23-0014), the Normandy Region project AGAC and the PAUSE Program.
\bigskip

\section*{References}

\medskip

\small

[1] A. Krizhevsky, I. Sutskever, and G. E. Hinton, “Imagenet classification with deep convolutional neural networks,” in NIPS, 2012.

[2] A.  Graves,  A.-r.  Mohamed,  and  G.  Hinton,  “Speech  recognition with deep recurrent neural networks,” in2013 IEEE international conference on acoustics, speech and signal processing.    IEEE, 2013, pp.6645–6649.

[3] F. Scarselli, M. Gori, A. C. Tsoi, M. Hagenbuchner, and G. Monfardini, “The graph neural network model,”IEEE Transactions on Neural Networks, vol. 20, no. 1, pp. 61–80, December 2009.

[4] J. Gilmer, S. S. Schoenholz, P. F. Riley, O. Vinyal, and G. E. Dahl,“Neural message passing from quantum chemistry,” in Proceedings of the International Conference on Machine Learning, 2017.

[5] M.  M.  Bronstein,  J.  Bruna,  Y.  LeCun,  A.  Szlam,  and  P.  Vandergheynst,  “Geometric  deep  learning:  Going  beyond  euclidean data,”IEEE  Signal  Processing  Magazine,  vol.  34,  no.  4,  pp.  18–42,July 2017.

[6] Z.  Wu,  S.  Pan,  F.  Chen,  G.  Long,  C.  Zhang,  and  P.  S.  Yu,  “A comprehensive survey on graph neural networks,” arXiv pr eprintarXiv:1901.00596, 2019.

[7] J.  Bruna,  W.  Zaremba,  A.  Szlam,  and  Y.  LeCun,  “Spectral  networks and locally connected networks on graphs,” arXiv preprint arXiv:1312.6203, 2013.

[8] M. Defferrard, X. Bresson, and P. Vandergheynst, “Convolutional neural networks on graphs with fast localized spectral filtering,”in Advances  in  Neural  Information  Processing  Systems,  2016,  pp.3844–3852.

[9] R. Levie, F. Monti, X. Bresson, and M. M. Bronstein, “Cayleynets:Graph convolutional neural networks with complex rational spectral  filters,”IEEE  Transactions  on  Signal  Processing,  vol.  67,  no.  1,pp. 97–109, Jan 2019.

[10]  T.  N.  Kipf  and  M.  Welling,  “Semi-supervised  classification  with graph   convolutional   networks,”   in International   Conference   on Learning Representations (ICLR), 2017.

[11]  P.  Velickovic,  G.  Cucurull,  A.  Casanova,  A.  Romero,  P.  Lio,  and Y. Bengio, “Graph attention networks,” in International Conference on Learning Representations (ICLR), 2018.

[12]  Z.  Yang,  W.  W.  Cohen,  and  R.  Salakhutdinov,  “Revisiting  semi-supervised learning with graph embeddings,” in Proceedings of the33rd International Conference on International Conference on Machine Learning, ICML’16, 2016.

[13]  Z.  Wu,  S.  Pan,  F.  Chen,  G.  Long,  C.  Zhang,  and  P.  S.  Yu,  “A comprehensive survey on graph neural networks,”arXiv preprint arXiv:1901.00596, 2019.

[14]  F. Chung,Spectral graph theory.American Mathematical Society,1997.

[15]  D.  I.  Shuman,  S.  K.  Narang,  P.  Frossard,  A.  Ortega,  and  P.  Vandergheynst, “The emerging field of signal processing on graphs:Extending high-dimensional data analysis to networks and other irregular domains,”IEEE signal processing magazine, vol. 30, no. 3,pp. 83–98, 2013.

[16]  M. Henaff, J. Bruna, and Y. LeCun, “Deep convolutional networks on graph-structured data,”arXiv preprint arXiv:1506.05163, 2015.

[17]  S.  Kearnes,  K.  McCloskey,  M.  Berndl,  V.  Pande,  and  P.  Riley,“Molecular  graph  convolutions:  moving  beyond  fingerprints,”Journal  of  computer-aided  molecular  design,  vol.  30,  no.  8,  pp.  595–608, 2016.

[18]  D.  K.  Duvenaud,  D.  Maclaurin,  J.  Iparraguirre,  R.  Bombarell,T. Hirzel, A. Aspuru-Guzik, and R. P. Adams, “Convolutional networks on graphs for learning molecular fingerprints,” in Advances in Neural Information Processing Systems, 2015, pp. 2224–2232.

[19]  M. Niepert, M. Ahmed, and K. Kutzkov, “Learning convolutional neural  networks  for  graphs,”  in Proceedings  of  the  International Conference on Machine Learning, 2016, pp. 2014–2023.

[20]  Y.  Hechtlinger,  P.  Chakravarti,  and  J.  Qui,  “A  generalization  of convolutional  neural networks  to  graph-structured  data,”  arXiv preprint arXiv:1704.08165, 2017.

[21]  J.  Atwood  and  D.  Towsley,  “Diffusion-convolutional  neural  networks,” in Advances in Neural Information Processing Systems, 2016,pp. 1993–2001.

[22]  W. Hamilton, Z. Ying, and J. Leskovec, “Inductive representation learning  on  large  graphs,”  in Advances  in  Neural  Information  Processing Systems, 2017, pp. 1024–1034.

[23]  F. Monti, D. Boscaini, J. Masci, E. Rodola, J. Svoboda, and M. M.Bronstein,  “Geometric  deep  learning  on  graphs  and  manifolds using  mixture  model  cnns,”  in Proceedings  of  the  IEEE  Conference on Computer Vision and Pattern Recognition, 2017, pp. 5115–5124.

[24]  M.  Fey,  J.  Eric  Lenssen,  F.  Weichert,  and  H.  Muller,  “Splinecnn:Fast  geometric  deep  learning  with  continuous  b-spline  kernels,”in Proceedings of the IEEE Conference on Computer Vision and Pattern Recognition, 2018, pp. 869–877.

[25]  B.  Xu,  N.  Wang,  T.  Chen,  and  M.  Li,  “Empirical  evaluation of  rectified  activations  in  convolutional  network,”arXiv  preprint arXiv:1505.00853, 2015.

[26]  F. Monti, O. Shchur, A. Bojchevski, O. Litany, S. Gunnemann, andM.  M.  Bronstein,  “Dual-primal  graph  convolutional  networks,”2018.

[27]  D. K. Hammond, P. Vandergheynst, and R. Gribonval, “Waveletson  graphs  via  spectral  graph  theory,”Applied  and  Computational Harmonic Analysis, vol. 30, no. 2, pp. 129–150, 2011.

[28]  M.  Zhang,  Z.  Cui,  M.  Neumann,  and  Y.  Chen,  “An  end-to-end deep  learning  architecture  for  graph  classification,”  in Thirty-Second AAAI Conference on Artificial Intelligence, 2018.

[29]  S.   Abu-El-Haija,   B.   Perozzi,   A.   Kapoor,   H.   Harutyunyan,N. Alipourfard, K. Lerman, G. V. Steeg, and A. Galstyan, “Mixhop:Higher-order graph convolution architectures via sparsified neighborhood  mixing,”  in International  Conference  on  Machine  Learning(ICML), 2019.

[30]  F.  Chollet,  “Xception:  Deep  learning  with  depthwise  separable convolutions,”  in Proceedings  of  the  IEEE  conference  on  computer vision and pattern recognition, 2017, pp. 1251–1258.

[31]  M. Sandler, A. G. Howard, M. Zhu, A. Zhmoginov, and L. Chen,“Mobilenetv2: Inverted residuals and linear bottlenecks,” in2018IEEE Conference on Computer Vision and Pattern Recognition, CVPR2018, Salt Lake City, UT, USA, June 18-22, 2018.IEEE Computer Society, 2018, pp. 4510–4520.

[32]  R.  Levie,  E.  Isufi,  and  G.  Kutyniok,  “On  the  transferability  of spectral graph filters,”arXiv preprint arXiv:1901.10524, 2019.

[33]  B.  Knyazev,  X.  Lin,  M.  R.  Amer,  and  G.  W.  Taylor,  “Spectral multigraph networks for discovering and fusing relationships in molecules,”arXiv preprint arXiv:1811.09595, 2018.

[34]  M.  Zitnik  and  J.  Leskovec,  “Predicting  multi cellular  function through   multi-layer   tissue   networks,”Bioinformatics,   vol.   33,no. 14, pp. i190–i198, 2017.

[35]  K.  Kersting,  N.  M.  Kriege,  C.  Morris,  P.  Mutzel,  and  M.  Neumann,  “Benchmark  data  sets  for  graph  kernels,”  2016,  http://graphkernels.cs.tu-dortmund.de.

[36]  Z. Ying, J. You, C. Morris, X. Ren, W. Hamilton, and J. Leskovec,“Hierarchical  graph  representation  learning  with  differentiable pooling,”  in Advances  in  Neural  Information  Processing  Systems,2018, pp. 4800–4810.

[37]  C.   Cangea,   P.   Velickovic,   N.   Jovanovic,   T.   Kipf,   and   P.   Lio,“Towards  sparse  hierarchical  graph  classifiers,”arXiv  preprintarXiv:1811.01287, 2018.

[38]  Y. S. Ting Chen, Song Bian, “Dissecting graph neural networks on graph classification,”CoRR, vol. abs/1905.04579, 2019.

[39]  K.  Xu,  W.  Hu,  J.  Leskovec,  and  S.  Jegelka,  “How  powerful  are graph  neural  networks?”  in International  Conference  on  Learning Representations, 2019.

[40]  J.  Zhang,  X.  Shi,  J.  Xie,  H.  Ma,  I.  King,  and  D.-Y.  Yeung,  “Gaan:Gated attention networks for learning on large and spatiotemporal graphs,” in Conference on Uncertainty in Artificial Intelligence, UAI,2018

\clearpage

\begin{@twocolumnfalse}
 \Large

 \begin{center}
 \noindent APPENDICES 
 \end{center}
 \end{@twocolumnfalse}



\setcounter{section}{0}
\def\thesection{\Alph{section}}
\setcounter{table}{0}
\renewcommand{\thetable}{A\arabic{table}}


\section{Theoretical Analysis of Chebyshev Kernels Frequency Profile}
\label{section:chebdetails}

In this appendix, we provide the expressions of the full and standard frequency profiles of the Chebyshev convolution kernels.

\begin{theorem}
  \label{Th:thch1}
The frequency profile of the first Chebyshev convolution kernel for any undirected arbitrary graph defined by $C^{(1)}=I$ can be defined by
\begin{equation}
  \label{eq:Eq1_ap1}
  \digamma_{1}(\blambda)={\bf1},
\end{equation}
where {\bf1} denotes the vector of ones of appropriate size.
\end{theorem}

\begin{proof}
When the identity matrix is used as convolution kernel, it just directly transmits the inputs to the outputs without any modification. This process is called all-pass filter. 
Mathematically, we can calculate the full frequency profile for kernel $I$ by using Corollary~\ref{cor:frequency_profile}, namely
\begin{equation}
  \label{eq:Eq2_ap1}
  \digamma_{1} =  U^\top I U = U^\top U = I,
\end{equation}
since the eigenvectors are orthonormal. 
Therefore, we can parameterize the diagonal of the full frequency profile by $\blambda$ and reach the standard frequency profile as follows:
\begin{equation}
  \label{eq:Eq3_ap1}
  \digamma_{1}(\blambda)=diag(I)={\bf1}.
\end{equation}
\end{proof}

\begin{theorem}
  \label{Th:thch2}
The frequency profile of the second Chebyshev convolution kernel for any undirected arbitrary graph given by $C^{(2)}=2L/\lambda_{\max}-I$ 
can be defined by
\begin{equation}
  \label{eq:Eq4_ap1}
  \digamma_{2}(\blambda)=\frac{2\blambda}{\lambda_{\max}}-{\bf1}.
\end{equation}
\end{theorem}

\begin{proof}
We can compute the $C^{(2)}$ kernel full frequency profile using Corollary~\ref{cor:frequency_profile}:
\begin{equation}
  \label{eq:Eq5_ap1}
  \digamma_{2} =  U^\top \left(\frac{2}{\lambda_{\max}}L-I \right) U.
\end{equation}
Since $U^\top I U=I$, \eqref{eq:Eq5_ap1} can be rearranged as
\begin{equation}
  \label{eq:Eq6_ap1}
  \digamma_{2} = \frac{2}{\lambda_{\max}} U^\top L U - I.
\end{equation}

Since $\blambda=[\lambda_1,\dots,\lambda_n]$ are the eigenvalues of the graph Laplacian $L$, those must conform to the following condition:
\begin{eqnarray}
  \label{eq:Eq8_ap1}
  LU &=&  U diag(\blambda); \\
  U^\top L U &=& diag(\blambda).
  \label{eq:Eq8_ap11}
\end{eqnarray}
Replacing \eqref{eq:Eq8_ap11} into \eqref{eq:Eq6_ap1}, we get
\begin{equation}
  \label{eq:Eq9_ap1}
  \digamma_{2} = \frac{2}{\lambda_{\max}} diag(\blambda) - I.
\end{equation}

This full frequency profile consists of two parts, a diagonal matrix 
and the negative identity matrix. 
Therefore, we can parameterize the full frequency matrix diagonal to show the standard frequency profile as follows:
\begin{equation}
  \label{eq:Eq13_ap1}
  \digamma_{2}(\blambda)=diag(\digamma_{2})=\frac{2\blambda}{\lambda_{\max}}-{\bf1}.
\end{equation}
\end{proof}

\begin{theorem}
  \label{Th:thch3}
The frequency profile of third and followings Chebyshev convolution kernels for any undirected arbitrary graph can be defined by
\begin{equation}
  \label{eq:Eq1_thch3}
  \digamma_{k}=2\digamma_{2}\digamma_{k-1} - \digamma_{k-2},
\end{equation}
and their standard frequency profiles by
\begin{equation}
  \digamma_{k}(\blambda)=2\digamma_{2}(\blambda)\digamma_{k-1}(\blambda) - \digamma_{k-2}(\blambda).
\end{equation}

\end{theorem}

\begin{proof}
Given the third and following Chebyshev kernels defined by $C^{(k)}=2C^{(2)}C^{(k-1)} - C^{(k-2)}$ and  using Corollary~\ref{cor:frequency_profile}, the corresponding frequency profile is
\begin{equation}
  \label{eq:Eq11_thch3}
  \digamma_{k}=U^\top\left( 2C^{(2)}C^{(k-1)} - C^{(k-2)}\right) U.
\end{equation}
By expanding \eqref{eq:Eq11_thch3}, we get
\begin{equation}
  \label{eq:Eq11_et_demi_thch3}
  \digamma_{k}=2U^\top C^{(2)}C^{(k-1)}U - U^\top C^{(k-2)} U.
\end{equation}
Since $UU^\top=I$, we can insert the product $UU^\top$ into \eqref{eq:Eq11_et_demi_thch3}. Thus, we have
\begin{align}
  \label{eq:Eq12_thch3}
  \digamma_{k} &= 2 U^\top C^{(2)}UU^\top C^{(k-1)}U - U^\top C^{(k-2)} U \\
  \digamma_{k} &= 2 \left(U^\top C^{(2)}U\right) \left(U^\top C^{(k-1)}U\right) - U^\top C^{(k-2)} U.
  \label{eq:Eq12_thch3_deux}
\end{align}
Since $\digamma_{k'}=U^\top C^{(k')}U$ for any $k'$, it follows that \eqref{eq:Eq12_thch3_deux} and \eqref{eq:Eq1_thch3} are identical.

Hence $\digamma_{1}$ and $\digamma_{2}$ are diagonal matrices, and the rest of the kernels frequency profiles become diagonal matrices in \eqref{eq:Eq1_thch3}. Therefore, we can write the corresponding standard frequency profiles of third and followings Chebyshev convolution kernels as follows:
\begin{equation}
  \label{eq:Eqchstdfreq}
  \digamma_{k}(\blambda)=2\digamma_{2}(\blambda)\digamma_{k-1}(\blambda) - \digamma_{k-2}(\blambda).
\end{equation}

\end{proof}

\section{Theoretical Analysis of CayleyNet Frequency Profile}
\label{section:cayleydetails}


CayleyNet uses in \eqref{eq:Eq4} the weight vector parametrization $F_{i,j,l}=[g_{i,j,l}(\lambda_1,h),...,g_{i,j,l}(\lambda_n,h)]^{\top}$, where the function $g(\cdot,\cdot)$ is defined in [9] by
\begin{equation}
  \label{eq:Eq1_apcy}
  g(\lambda,h)=c_0 + 2 Re\left( \sum_{k=1}^{r}c_k \left(\frac{h\lambda-\imath}{h\lambda+\textbf{i}}\right)^k  \right),
\end{equation}
where $\textbf{i}^2=-1$, $Re(\cdot)$ is the function that returns the real part of a given complex number, $c_0$ is a trainable real coefficient, and $c_1,\ldots,c_r$ are complex trainable coefficients. We can write $h\lambda-\textbf{i}$ in Euler form by $\sqrt{h^2\lambda^2+1}. e^{\textbf{i}\atantwo(-1,h\lambda)}$ and for $h\lambda+\textbf{i}$ by $\sqrt{h^2\lambda^2+1}. e^{\textbf{i}\atantwo(1,h\lambda)}$. By this substitution, \eqref{eq:Eq1_apcy} becomes
\begin{equation}
  \label{eq:Eq2_apcy}
  g(\lambda,h)=c_0+2 Re\left( \sum_{k=1}^{r}c_k e^{\textbf{i}k\left(\atantwo(-1,h\lambda)-\atantwo(1,h\lambda)\right)}  \right).
\end{equation}
where $\atantwo(y,x)$ is the inverse tangent function, which finds the angle (in range of $[-\pi, \pi]$) of a point given its $y$ and $x$ coordinates. For further simplification, let us introduce the $\theta(\cdot)$ function defined by
\begin{equation}
  \label{eq:Eq2a_apcy}
  \theta(x)=\atantwo(-1,x)-\atantwo(1,x).
\end{equation}
Since the $c_k$s are complex numbers, we can write them as a sum of real and imaginary parts, $c_k=a_k/2+\textbf{i}b_k/2$ (the scale factor 2 is added for convenience). Thus, \eqref{eq:Eq2_apcy} can be rewritten as follows:
\begin{equation}
  \label{eq:Eq3_apcy}
  g(\lambda,h)=c_0+ Re\left( \sum_{k=1}^{r}(a_k+\textbf{i}b_k) e^{\textbf{i}k\theta(h\lambda)}\right).
\end{equation}
We can replace $e^{\textbf{i}k\theta(h\lambda)}$ with its polar coordinate equivalence form  $\cos(k\theta(h\lambda)) + \textbf{i}\sin(k\theta(h\lambda))$. When we remove the imaginary components because of $Re(\cdot)$ function, \eqref{eq:Eq3_apcy} becomes
\begin{equation}
  \label{eq:Eq4_apcy}
  g(\lambda,h)=c_0+ \sum_{k=1}^{r} a_k\cos(k\theta(h\lambda)) -b_k\sin(k\theta(h\lambda)).
\end{equation}
In this definition, there is no complex coefficient, but only real coefficients ($c_0$, $a_k$ and $b_k$ for $k=1, \ldots, r$) to be tuned by training. 
By using the form in \eqref{eq:Eq4_apcy}, we can  parametrize CayleyNet by the parametrization matrix $B \in \mathbb{R}^{n \times 2r+1}$, as in \eqref{eq:Eq4b}, by
\begin{equation}
  \label{eq:Eq5_apcy}
  [g(\lambda_0,h), \dots ,g(\lambda_n,h)]^{\top}=B[c_0,a_1,b_1,\dots,a_r,b_r]^{\top}.
\end{equation}
The $s$-th column vector of matrix $B$, denotes $B_s$, must fulfill the following conditions:
\begin{equation}
  \label{eq:Eq6_apcy}
  B_s=\digamma_s(\blambda) = \left\{
        \begin{array}{l@{\text{ if }}l}
            {\bf1} & s = 1 \\           \cos(\frac{s}{2}\theta(h\blambda)) & s \in \{2,4,\dots, 2r\} \\            -\sin(\frac{s-1}{2}\theta(h\blambda)) & s \in \{3,5,\dots, 2r+1\}
        \end{array}
    \right. 
\end{equation}
We can see CayleyNet as a spectral graph convolution that uses $2r+1$ convolution kernels. The first kernel is an all-pass filter, and  the frequency profiles of remaining $2r$ kernels ($\digamma_s(\blambda)$) are created using sine and cosine functions, with a parameter $h$ used to scale the eigenvalues in~\eqref{eq:Eq6_apcy}. Considering \eqref{eq:Eq9a} in Theorem~\ref{Th:th1}, we can write CayleyNet's convolutions ($C^{(s)}$) in spatial domain. CayleyNet includes the tuning of this scaling parameter in the  training pipeline. Note that because of the function definition in \eqref{eq:Eq2a_apcy}, $\theta(h\lambda)$ is not linear in $\lambda$. Therefore, $\digamma_s$ cannot be a perfect sinusoidal in $\lambda$s. 

\section{Theoretical Analysis of GCN Frequency Profile}
\label{section:gcndetails}

In this appendix, we study the GCN and its convolution kernel. We start by deriving the expression of its frequency profile.

\begin{theorem}
  \label{Th:th2}
The frequency profile of GCN convolution kernel is defined by
\begin{equation}
  \label{eq:Eq1_ap}
  C_{GCN} = \widetilde{D}^{-1/2}\widetilde{A} \widetilde{D}^{-1/2},
\end{equation}
 and can be written as
\begin{equation}
  \label{eq:Eq2_ap}
  \digamma_{GCN}(\blambda)= {\bf1}-\frac{p}{p+1}\blambda,
\end{equation}
where $\blambda$ is the eigenvalues of the normalized graph Laplacian and the given graph is an undirected regular graph whose node degrees are all equal to $p$.
\end{theorem}

\begin{proof}
Since $\widetilde{D}_{i,i} = \sum_j \widetilde{A}_{i,j}$ and $\widetilde{A} = ( A + I)$, we can rewrite \eqref{eq:Eq1_ap} as:
\begin{equation}
  \label{eq:Eq3_ap}
  C_{GCN} = (D+I)^{-1/2}(A+I) (D+I)^{-1/2}.
\end{equation}
Under the assumption that all node degrees are equal to $p$, we can write the diagonal degree matrix by $D=pI$. Then, \eqref{eq:Eq3_ap} can be rewritten as 
\begin{equation}
  \label{eq:Eq4_ap}
  C_{GCN} = ((p+1)I)^{-1/2}(A+I) ((p+1)I)^{-1/2},
\end{equation}
which is equivalent to
\begin{equation}
  \label{eq:Eq5_ap}
  C_{GCN} = \frac{A+I}{p+1}.
\end{equation}
%
%
Using Corollary~\ref{cor:frequency_profile}, we can express the frequency profile of $C_{GCN}$ in matrix form by 
\begin{equation}
  \label{eq:Eq7_ap}
  \digamma_{GCN}= \frac{1}{p+1}U^{\top}A U + \frac{1}{p+1} I.
\end{equation}
Since $\blambda=[\lambda_1,\dots,\lambda_n]$ are the eigenvalues of the normalized graph Laplacian $L=I-D^{-1/2}AD^{-1/2}$, they must conform to the following condition:
\begin{equation}
  \label{eq:Eq8_ap}
  \left(I-D^{-1/2}AD^{-1/2} \right)  U=U diag(\blambda).
\end{equation}
\noindent According to 
$D=pI$, it conforms to $D^{-1/2}AD^{-1/2}={A}/{p}$. Thus, \eqref{eq:Eq8_ap} can be written as
\begin{equation}
  \label{eq:Eq9_ap}
  U-\frac{A U}{p} =U diag(\blambda).
\end{equation}
\noindent Then $AU$ is expressed as
\begin{equation}
  \label{eq:Eq10_ap}
  A U =pU-pU diag(\blambda)
\end{equation}
Replacing $A U$ in \eqref{eq:Eq7_ap}, we obtain
\begin{equation}
  \label{eq:Eq11_ap}
  \digamma_{GCN}= \frac{1}{p+1}U^{\top} \left( pU -pU diag(\blambda)  \right)
  + \frac{1}{p+1} I.
\end{equation}
\noindent Since $U^{\top}U=I$, then we have
\begin{equation}
  \label{eq:Eq12_ap}
  \digamma_{GCN}= \frac{pI-p diag(\blambda)+I}{p+1}.
\end{equation}
This expression can be simplified to
\begin{equation}
  \label{eq:Eq13_ap}
  \digamma_{GCN}= I- \frac{p}{p+1} diag(\blambda),
\end{equation}
which is equal to the matrix form defined in \eqref{eq:Eq2_ap} since $\digamma_{GCN}(\blambda) =  diag(\digamma_{GCN})$.
\end{proof}




\medskip

This demonstration shows that the GCN frequency profile acts as a low-pass filter. When the given graph is a circular undirected graph, all node degrees are equal to $p=2$, leading to a frequency profile defined by ${\bf1}-2\blambda/3$. Since the normalized graph Laplacian eigenvalues are in the range $[0,2]$, the filter magnitude linearly decreases until the third quarter of the spectrum (cut-off frequency) where it reaches zero. Then it linearly increases until the end of the spectrum. This explains the shape of the frequency profile of GCN convolutions for 1D regular graph observed in \figurename~\ref{fig:gcnnfreq}. 

However, this conclusion cannot explain the perturbations on the GCN frequency profile. To analyse this point, we relax the assumption $D=pI$ and rewrite \eqref{eq:Eq3_ap}  
as
\begin{equation}
  \label{eq:Eq11_ap_conv}
  C_{GCN} = (D+I)^{-1}+(D+I)^{-1/2}A (D+I)^{-1/2}.
\end{equation}
We can see that the GCN kernel consists of two parts, $C_{GCN} =c_1+c_2$, where first part is given by $c_1=(D+I)^{-1}$ and the second one is $c_2=(D+I)^{-1/2}A (D+I)^{-1/2}$. 

For the second part ($c_2$), we can write it using the element-wise multiplication operator $\odot$ (Hadamard multiplication)
\begin{equation}
  \label{eq:Eq12_ap_conv}
   c_2 =A \odot \sqrt{{\bf1}/(d+1)} \cdot \sqrt{{\bf1}/(d+1)}^{\top} ,
\end{equation}
where $d$ is the column degree vector $d=diag(D)$ and the division and square-root are also element-wise (Hadamard) operations. With the same notation, we can rewrite the Chebyshev second kernel, assuming that $\lambda_{\max}=2$,
\begin{equation}
  \label{eq:Eq13_ap_conv}
  C^{(2)} = - A \odot \sqrt{{\bf1}/d} \cdot\sqrt{{\bf1}/d}^{\top}.
\end{equation}
The two expressions \eqref{eq:Eq12_ap_conv} and \eqref{eq:Eq13_ap_conv} show that negative $c_2$ is an approximation of the second Chebyshev kernel if  vector $d$ consists of same values, as it was assumed in Theorem~\ref{Th:th2}. When the vector $d$ is composed of different values, the two matrices $\sqrt{{\bf1}/d}.\sqrt{{\bf1}/d}^{\top}$ and  $\sqrt{{\bf1}/(d+1)}.\sqrt{{\bf1}/(d+1)}^{\top}$ are not proportional for each coordinate (i.e., entry). To obtain $c_2$ from $C^{(2)}$, we need to use different coefficients for each coordinate of the kernel. If the difference between node degrees is important, these coefficients have the strong influence, and $c_2$ may be very different from $C^{(2)}$. Conversely, if the node degrees are quite uniform, these coefficients may be neglected. This phenomenon is the first cause of perturbation on GCN frequency profile.

The first part ($c_1$) of the GCN kernel in \eqref{eq:Eq11_ap_conv} is more interesting. Actually, it is a diagonal matrix that shows the contribution of each node in the convolution process. Instead of looking for some approximations of known frequency profiles such as those of Chebyshev kernels, we can write its frequency profile directly. Using Corollary~\ref{cor:frequency_profile}, we can express the frequency profile of $c_1$ in matrix form by
\begin{equation}
  \label{eq:Eq14_ap}
  \digamma_{c_1} = ( U^\top c_1 \, U ),
\end{equation}
where $U$ is the eigenvectors matrix. By taking advantage of having a diagonal kernel $c_1$, we can express each component of full frequency profile as
\begin{equation}
  \label{eq:Eq15_ap}
  \digamma_{c_1}(i,j) = \sum_{k=1}^{n}\left( \frac{1}{1+d_k}U_{i,k}U_{j,k} \right),
\end{equation}
where $n$ is the number of nodes in the graph, $d_k$ is degree of the $k$-th node, $U_{i,k}$ is the $k$-th element of $i$-th eigenvector. As eigenvectors $U_i$ and $U_j$ are orthogonal for $i\neq j$, their scalar product is null. However, in \eqref{eq:Eq15_ap}, the weighting coefficient $\frac{1}{1+d_k}$ is not constant over all the dimensions of the eigenvectors. Therefore, there is no guarantee that $\digamma_{c_1}(i,j)$ is null. This is another reason that explains that the GCN frequency profile has many non-zero elements outside of the diagonal.


In addition, it is also clear that the standard frequency profile of $c_1$ (diagonal of $\digamma_{c_1}$, i.e., $\digamma_{c_1}(i,i)$ in \eqref{eq:Eq15_ap}) is not smooth. Indeed, the diagonal elements of $\digamma_{c_1}$ can be written as a weighted sum of squared eigenvalues elements, which again is weighted by $1/(1+d_k)$. If the latter is constant for all $k$, the sum of squared eigenvectors elements has to be 1 since the eigenvectors have unit L2-norm. But in the general case where $1/(1+d_k)$ are not necessarily constant over all the dimensions of eigenvectors, the diagonal of the matrix may have some perturbations. This point constitutes another explanation on the fact that the GCN standard frequency profile is not smooth.

On the other hand, under the assumption that the node degrees distribution is uniform, we can derive the following approximation:
\begin{equation}
  \label{eq:Eq16_ap}
  p \approx \overline{d}=\frac{1}{n}\sum_{k=1}^n{d_k}.
\end{equation}
We can then write an approximation of the GCN frequency profile as a function of the average node degree by replacing $p$ with $\overline{d}$ in \eqref{eq:Eq10_ap} and obtain the final approximation:
\begin{equation}
  \label{eq:Eq17_ap}
  \digamma_{GCN}(\blambda) \approx {\bf1}-\frac{\overline{d}}{\overline{d}+1}\blambda.
\end{equation}
We can theoretically show the cut-off frequency where GCN kernel's frequency profile reach 0 by
\begin{equation}
  \label{eq:Eq18_ap}
  \lambda_{\text{cut}} \approx \frac{\overline{d}+1}{\overline{d}}.
\end{equation}

\section{Application Details}
\label{section:appdetails}

This section presents some additional details on the experimental settings and parameter tuning.

In the experiments of transductive learning problems, we tuned single convolution low-pass filter's parameter $\eta$ in $\digamma_{1}(\blambda)=({\bf1}-\blambda/\lambda_{\max})^{\eta}$. \tablename~A1 shows searching space of $\eta$ and their loss and accuracy performance over different datasets. One can see that we decided to use $\eta=5$ for Cora and Citeseer, and $\eta=3$ for PubMed dataset where it maximizes the validation set performance.

\begin{table}[!h]
\renewcommand{\arraystretch}{1.3}
\caption{Minimum validation set loss value and maximum validation set accuracy over different low-pass filters.}
\centering
\begin{tabular}{c c c c c c c}
Convolution & \multicolumn{2}{c}{Cora} & \multicolumn{2}{c}{Citeseer} & \multicolumn{2}{c}{PubMed}  \\
$\digamma_{1}(\blambda)$ & Loss & Acc & Loss & Acc & Loss & Acc \\
\toprule
$({\bf1}-\blambda/\lambda_{\max})^1$  &1.116  &80.4 & 1.12 & 73.0   &0.654 &77.1 \\
$({\bf1}-\blambda/\lambda_{\max})^3$ &0.745 &81.8 &1.02 &72.6 &\textbf{0.572} &\textbf{81.1} \\
$({\bf1}-\blambda/\lambda_{\max})^5$ &\textbf{0.705} &\textbf{81.8} &\textbf{1.02} &\textbf{73.8}   &0.592 &80.5 \\
$({\bf1}-\blambda/\lambda_{\max})^{10}$ &0.752 &81.2 &1.01 &72.2 &- &- \\
$({\bf1}-\blambda/\lambda_{\max})^{20}$ &0.792 &80.8 &1.01 &71.2 &- &- \\

\bottomrule
\end{tabular}
\label{tabletdesign}

\end{table}

\begin{table}
\renewcommand{\arraystretch}{1.3}
\caption{Used hyperparameters.}
\scriptsize
\centering
\begin{tabular}{l c c c c c c c }
\hline
Hyperparameters & Cora & Citeseer & PubMed & PPI & PROTEINS & ENZYMES-label & ENZYMES-allfeat \\
\hline

Hidden Activations & ReLU & ReLU & ReLU & ReLU & ReLU & ReLU & ReLU \\

Output Activation & Linear & Linear & ReLU  & Linear & Linear & Linear & Linear \\

Hidden Biases & False & False & False & True & False & False & False  \\

Output Bias & True  & True  & False & True  & True & True  & True  \\

Input Dropout & 0.75 & 0.75 &0.25 & 0 & 0 & 0.1 & 0.1 \\
Kernel Dropout & 0.75 & 0.75 &0 & 0 & 0 & 0.1 & 0.1 \\
Weight Decay &3e-4 & 3e-4 & 5e-4 & 0 & 0 & 1e-4 & 1e-4 \\
Weight Decay on DSG & 3e-3 & 3e-3 & 5e-3 & 0 &- & - & - \\
Learning Coeff & 0.01 & 0.01 & 0.01 & 0.01 & 0.0005 & 0.001 & 0.001 \\
Batch Size & 1 & 1 & 1 & 1 & 333 & 180 & 180 \\
Epoch & 400 & 100 & 250 & 500 &100 &500 &500 \\
\hline
\end{tabular}
\label{tablehyper}
\end{table}

In depthwise separable graph convolution layer, the initialization of the trainable parameters $w^{(s,l)}$ affects the performance. If designed convolutions are supposed to have equal effect on the model, these parameters can be initialized randomly. But, if one is supposed to have more effect on the model, the important convolution kernel's correspondence weights can be initialized by 1, the rest of them initialized by 0. In our model, we assumed the first kernel is always the most important kernel. Thus, we initialized the first kernel's depthwise separable weights as $w^{(1,l)}=1$, and the rest of the kernel's depthwise separable weights $w^{(s,l)}=0$ when $s>1$. In this way, the model starts training as there is only kernel, which is supposed to be the most important one.

The used hyperparameters in our experiments are presented in Table~\ref{tablehyper}. We applied $\verb|softmax|$ to the output of the models and calculate cross entropy loss function for all problems expect PPI dataset. Since PPI is two class classification problem and we coded output by one neuron, we applied $\verb|tansig|$ to the output of the PPI model and used binary cross entropy as loss function. In our models we did not consider any regularization on the bias parameter, but we applied the L2-loss to the trainable weights. In the depthwise separable layer, there are two different kinds of weights where additional one is depthwise weights. That is why in \tablename~\ref{tablehyper}, there is two different weight decays. We always used $\verb|ReLU|$ activation on the hidden layers and the Linear for output layers
. The table also provides if bias values are used in the hidden and output layers
. In our model, we used two different types of dropout: the dropout applied on the inputs of the layer as usually used in the literature, and the dropout applied on the convolution kernel, which was first used in \cite{velivckovic2017graph} according to the best of our knowledge. Since Cora, Citeseer and PubMed datasets consist of one single graph, batch size is 1 for these problems. For the PPI dataset of only 24 graphs, we still prefer to update the model for each training graph. But for PROTEINS and ENZYMES datasets, we update the model 3 times in each epoch. Since in PROTEINS there are 1113 graphs and in each fold there are 1000 graphs in the train set, we used a 333 batch size. As the same in ENZYMES there is 540 graphs in each train fold, we used a 180 batch size to update the model 3 times in a single epoch. We used the Adam optimization and a fixed learning-coefficient in all models. The used learning coefficient and the maximum epoch number can be found in the Table~\ref{tablehyper}. 

\end{document}